\documentclass[runningheads]{llncs}
\usepackage[utf8]{inputenc}
\usepackage[table, dvipsnames]{xcolor}

\usepackage[final,year=2026]{eccv}
\usepackage{eccvabbrv}

\usepackage{graphicx}
\usepackage{pdfpages}
\usepackage{booktabs}
\usepackage{multirow}
\usepackage{amsmath,amssymb,mathtools}
\usepackage{bibunits}

\usepackage{pifont}
\definecolor{Gray}{gray}{0.85}
\definecolor{LightPurple}{RGB}{230, 230, 250}
\definecolor{GainColor}{RGB}{70, 90, 145}
\newcommand{\best}[1]{\textbf{#1}}
\newcommand{\gain}[1]{\textcolor{GainColor}{#1}}
\newcommand{\cmark}{{\color{ForestGreen}\ding{51}}}
\newcommand{\xmark}{{\color{red}\ding{55}}}
\makeatletter
  \newcommand\figcaption{\def\@captype{figure}\caption}
  \newcommand\tabcaption{\def\@captype{table}\caption}
\makeatother

\usepackage[accsupp]{axessibility}

\usepackage{mycommands}

\usepackage[breaklinks,colorlinks,citecolor=eccvblue,urlcolor=eccvblue]{hyperref}
\usepackage{orcidlink}

\newcommand{\R}{\mathbb{R}}

\begin{document}
\begin{bibunit}[splncs04]

\title{Fully Rotation-Equivariant Spectral-Spatial Learning for Multispectral Object Detection}
\titlerunning{Fully Rotation-Equivariant Spectral-Spatial Learning}
\author{Peng Zhang\orcidlink{0009-0008-9585-9622} \and Tingfa Xu$^\dagger$\orcidlink{0000-0001-5452-2662} \and Shuaihao Han\orcidlink{0009-0008-2594-1516} \and Jianan Li$^\dagger$\orcidlink{0000-0002-6936-9485}}
\authorrunning{P. Zhang et al.}
\institute{Beijing Institute of Technology, Beijing, China}
\maketitle
\begingroup
\renewcommand{\thefootnote}{\dag}
\footnotetext{Correspondence to: Tingfa Xu and Jianan Li.}
\endgroup

\begin{figure*}[t]
  \centering
  \includegraphics[width=\linewidth]{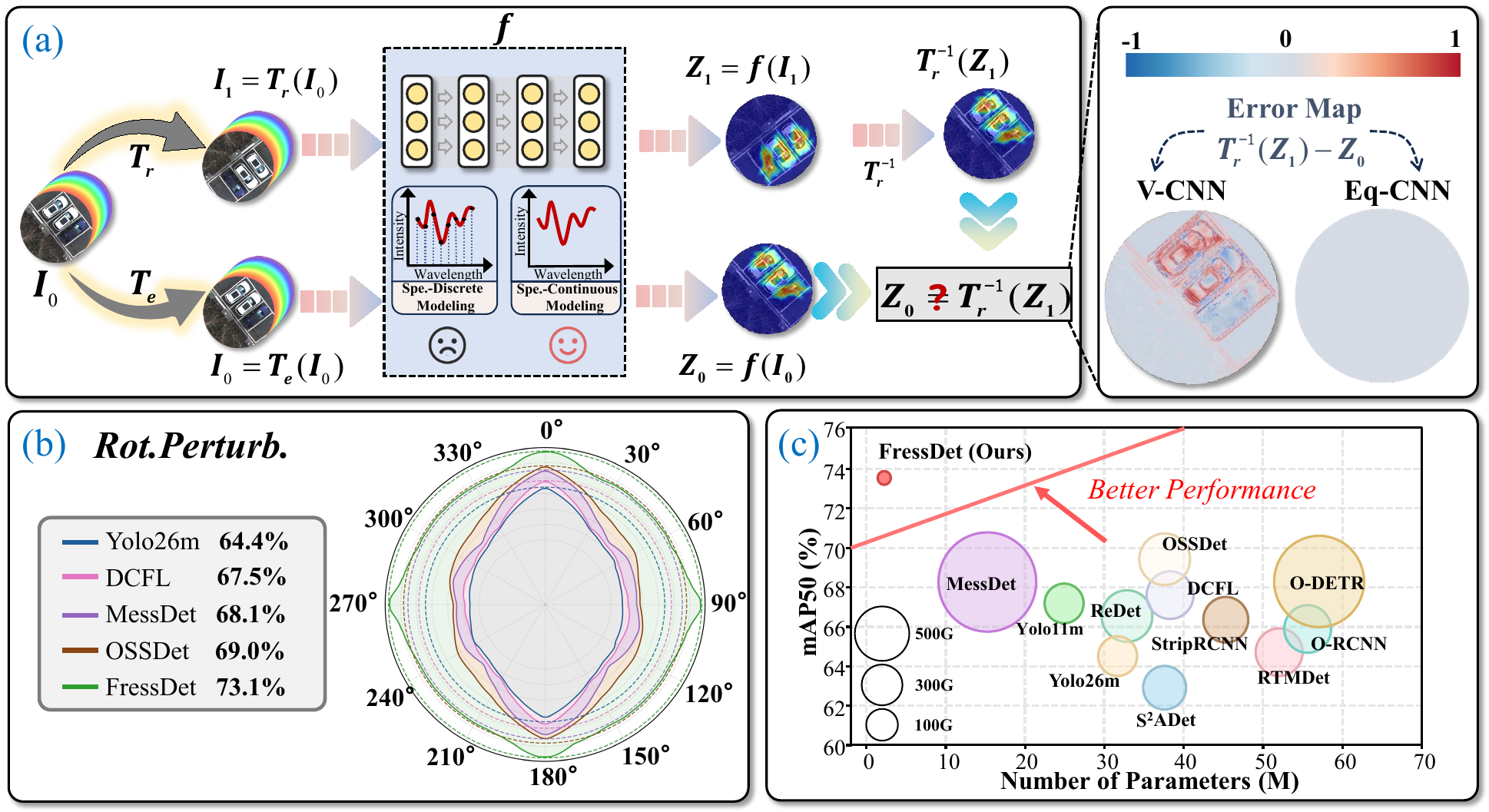}
  \caption{(a) Transformation consistency in vanilla (V-CNN) and equivariant (Eq-CNN) networks. $T_e$ and $T_r$ denote the identity and a rotation, $f$ is a feature extractor, and $Z_0$, $Z_1$ are features of $I_0\!=\!T_e(I_0)$ and $I_1\!=\!T_r(I_0)$. Ideally $T_r^{-1}(Z_1)\!\approx\!Z_0$; the error map confirms Eq-CNN yields lower transformation error. However, existing methods have not yet unified continuous spectral modeling with rotation-equivariant geometric priors. (b) Detection robustness under rotational perturbations at 1° intervals on MODA. FressDet exhibits the best robustness. (c) mAP50 vs.\ parameter count on MODA, where circle size is proportional to FLOPs. FressDet achieves the best performance with the fewest parameters and lowest FLOPs.}
  \label{fig:motivation}
\end{figure*}

\begin{abstract}
Existing multispectral detectors are limited by discrete spectral processing, a scale-dependent shift in the relative reliability of spectral and spatial cues across pyramid levels, and the lack of explicit rotation-equivariant geometric priors for arbitrarily oriented objects. To tackle these limitations, we propose FressDet, a fully rotation-equivariant spectral--spatial learning framework for multispectral object detection, capable of capturing the continuous, ordered nature of spectral structure and enabling reliable spectral--spatial fusion across pyramid levels under arbitrary in-plane rotations. FressDet integrates three complementary components. Spectral Implicit Warp (SpeIW) enables query-based spectral resampling via a coordinate-conditioned implicit field, yielding a monotone, order-preserving warp. Rotation-Equivariant Consistency Weighting (ReCoW) adaptively fuses spectral and spatial branches based on branch reliability, reinforcing informative cues while suppressing noise across pyramid levels. The oriented-aware head exploits group-indexed features to stably predict oriented objects without parameter replication. Taken together, FressDet learns more discriminative and robust spectral--spatial representations even under rotational perturbations. By achieving state-of-the-art performance with 93\% fewer parameters on five public benchmarks, FressDet demonstrates its effectiveness and generalizability.
\par\noindent Code is available at \url{https://github.com/Riiluo/FressDet}.
\end{abstract}

\section{Introduction}

Multispectral object detection (MOD) improves localization and recognition by jointly exploiting spatial structure and spectral cues. Compared with RGB imagery, multispectral images (MSIs) measure spectral responses at sampled wavelengths, yielding per-pixel spectral signatures with strong inter-band correlation. Prior methods either decouple spectral and spatial information through dimensionality reduction (e.g., PCA~\cite{pearson1901liii} and band selection~\cite{sun2019hyperspectral,zhang2024tensorial}) or rely on attention-based modeling to capture spectral--spatial interactions~\cite{han2025moda}. Although these methods achieve promising performance, two challenges remain in MOD:
(i) discrete spectral processing that neglects the continuous and ordered structure of the spectral dimension;
(ii) a scale-dependent shift in the relative reliability between spectral and spatial cues across pyramid levels, which can degrade fusion performance in complex scenes.

Rotation equivariance is a desirable geometric inductive bias that exploits the intrinsic rotational symmetry and arbitrary orientations of objects, improving representation stability and generalization across diverse visual tasks~\cite{cohen2016group,e2cnn}, including ground-level~\cite{wu2025gbias,wu2025r2det} and aerial~\cite{wu2025measuring,redet} object detection as well as semantic segmentation~\cite{xu2025precm}, yet many multispectral detectors still rely on heavy rotation-based data augmentation in practice. However, extending this inductive bias to MOD requires rethinking the modeling paradigm. In particular, the continuous, ordered spectral structure of MSIs is mismatched with most equivariant designs built on discrete sampling (Fig.~\ref{fig:motivation} (a)). Moreover, reliability-aware pyramid fusion under equivariant architectures is non-trivial in complex scenes. Furthermore, group-indexed interactions increase the computational burden. Therefore, a key challenge in MOD is to explicitly model spectral continuity and enable reliable spectral--spatial fusion across scales, while enforcing rotation-equivariant geometric priors with manageable overhead.

To address the challenges of spectral discreteness, scale-dependent fusion reliability, and the lack of rotation-equivariant geometric priors, we propose FressDet, the first fully rotation-equivariant spectral--spatial learning framework for multispectral object detection. FressDet captures continuous spectral structure, enables reliability-aware pyramid fusion across levels, and stably predicts oriented objects without parameter replication.

To overcome the discreteness of spectral processing, we draw on the universal approximation theorem~\cite{cybenko1989sigmoidal,hornik1991approximation}: implicit neural representations (INRs) parameterize signals as coordinate-conditioned neural functions, enabling approximation of continuous functions with arbitrarily small error. Accordingly, we introduce Spectral Implicit Warp (SpeIW), an efficient spectral operator that models multispectral representations as a coordinate-conditioned function queried along the spectral dimension with Fourier-encoded spectral coordinates~\cite{tancik2020fourierfeatures,rahaman2019spectralbias}. This encoding mitigates the spectral bias of coordinate-based networks, enabling the representation of higher-frequency spectral variations. Specifically, we parameterize spectral warping over rotation-equivariant features via a factorization into coordinate-dependent bases and feature-conditioned coefficients. The bases encode regularities along the spectral dimension, while the coefficients provide pixel-wise modulation of the warping offsets. By predicting strictly positive spectral steps and integrating them cumulatively, we obtain an order-preserving mapping by construction.

Addressing scale-dependent reliability shifts across pyramid levels, we propose Rotation-Equivariant Consistency Weighting (ReCoW), a lightweight refinement module with dual prototype-driven routing. The spectral branch uses soft assignment to capture discriminative spectral patterns, while the spatial branch uses hard grouping to improve geometric coherence. Their rotation-invariant agreement governs an equivariant residual refinement of the input feature, enabling adaptive spectral--spatial integration across pyramid levels.

For arbitrarily oriented objects, we propose an oriented-aware head that exploits feature maps with explicit orientation indices, enabling stable prediction across orientations without parameter replication. This design enhances feature utilization and substantially reduces the parameter count.

Taken together, these components enable FressDet to learn more discriminative and robust spectral--spatial representations. Extensive experiments on five public benchmarks demonstrate that FressDet achieves state-of-the-art accuracy with 93\% fewer parameters than the previous best method (Fig.~\ref{fig:motivation} (c)) and the best robustness even under input rotational perturbations (Fig.~\ref{fig:motivation} (b)). Our contributions are summarized as follows:
\begin{itemize}
  \item[$\bullet$] Present FressDet, the first fully rotation-equivariant spectral--spatial learning framework unifying continuous spectral modeling, reliability-aware pyramid fusion, and stable orientation-aware prediction.
  \item[$\bullet$] Introduce Spectral Implicit Warp for continuous, order-preserving spectral resampling via a coordinate-conditioned implicit field.
  \item[$\bullet$] Propose Rotation-Equivariant Consistency Weighting, which performs adaptive, reliability-aware spectral--spatial fusion across pyramid levels.
  \item[$\bullet$] Design an oriented-aware head for stably predicting arbitrarily oriented objects without parameter replication.
\end{itemize}

\section{Related Work}

\nbf{Multispectral Object Detection}
Multispectral imagery provides spectral cues beyond RGB data, facilitating more discriminative object representations. Existing methods either separate spectral and spatial streams with cross-branch fusion~\cite{yan2021object,hod3k} or integrate them in a unified design~\cite{han2025moda}. However, most treat spectral features independently, leaving the continuous and ordered nature of spectral structure insufficiently modeled, while scale-dependent shifts in spectral--spatial reliability across pyramid levels remain unaddressed. In contrast, our method explicitly models spectral continuity and enables reliability-aware spectral--spatial fusion across pyramid levels.

\nbf{Rotation-equivariant Networks}
Since Cohen~\etal~\cite{cohen2016group} introduced G-CNNs, rotation-equivariant networks have shown strong performance across diverse vision tasks. In object detection, ReDet~\cite{redet} pioneered rotation-equivariant aerial detectors, and subsequent works~\cite{wu2025measuring,wu2025r2det} advanced equivariant architectures. However, these methods are only approximately equivariant. Lee~\etal~\cite{fred} present FRED, the first fully rotation-equivariant detector with deformable convolutions, though its group-indexed interactions incur considerable computational overhead. We develop the first fully rotation-equivariant spectral--spatial learning framework for multispectral object detection, integrating spatial geometry with discriminative spectral cues while maintaining manageable overhead.

\nbf{Implicit Neural Representations}
Implicit neural representations (INRs) parameterize signals as coordinate-to-value neural fields, achieving success in 3D scene modeling~\cite{nerf} and continuous image representation~\cite{liif,ultrasr}. In spectral imaging, INRs have been applied to spectral rendering~\cite{spectralnerf} and multispectral--hyperspectral fusion~\cite{feinfn}. However, prior INR methods primarily address reconstruction or fusion tasks, whereas multispectral detection requires modeling spectral continuity under a strict order-preserving constraint. Our method addresses this by introducing an order-preserving spectral warping mechanism that maintains rotation equivariance.
\section{Method}
In this section, we introduce the preliminaries on rotation equivariance and describe how the components of FressDet are assembled into a fully rotation-equivariant pipeline (Fig.~\ref{fig:network}).
A detailed equivariance analysis is provided in the supplementary material.

\begin{figure*}[t]
  \centering
  \includegraphics[width=\linewidth]{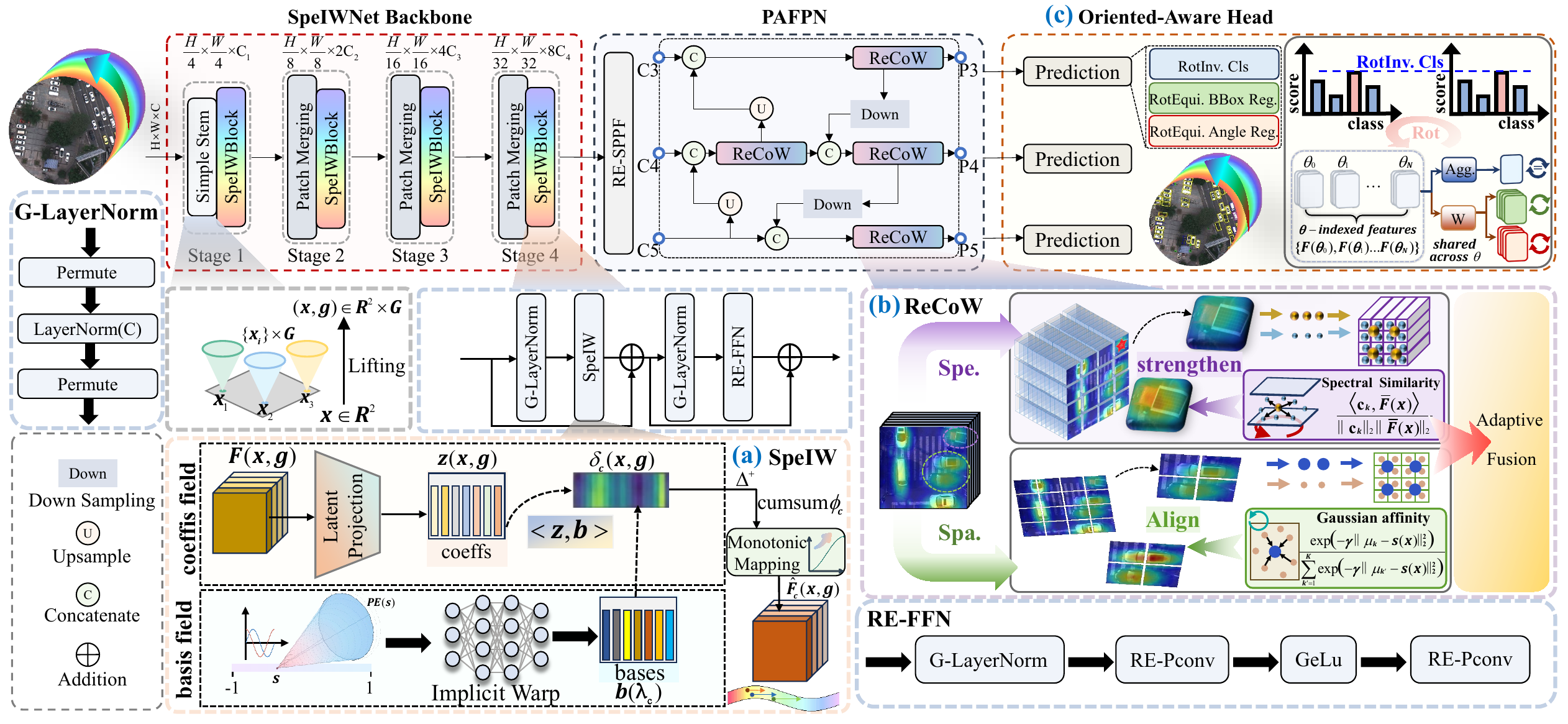}
  \caption{The overall framework of FressDet. (a) SpeIW resamples spectral features via a continuous, order-preserving warp. (b) ReCoW improves spectral--spatial fusion across pyramid levels through complementary routing and consistency weighting. (c) The oriented-aware head stably predicts oriented objects without parameter replication.}
  \label{fig:network}
\end{figure*}
\subsection{Preliminaries}
\label{sec:prelim}

\nbf{Symmetry Groups and Equivariance}
We consider equivariance to discrete in-plane rotations within the Euclidean group $\mathrm{E}(2)=\mathbb{R}^{2}\rtimes \mathrm{O}(2)$, the semi-direct product of translations and orthogonal transformations.
By restricting to the cyclic subgroup $\mathrm{C}_N=\{r_k \mid k=0,\dots,N{-}1\}$ of $N$ equally spaced rotations, we sample features at $N$ discrete orientations.
We denote this group by $G$ ($|G|=N$) in the sequel.
A function $\Phi\colon X\!\to\!Y$ is equivariant with respect to $G$ if:
\begin{equation}
  \Phi\!\bigl(\rho_{g}^{X}(x)\bigr)
  = \rho_{g}^{Y}\!\bigl(\Phi(x)\bigr),
  \quad \forall\, g\in G,\; x\in X,
  \label{eq:equiv}
\end{equation}
where $\rho_{g}^{X}$ and $\rho_{g}^{Y}$ denote group actions on $X$ and $Y$, respectively.
Intuitively, equivariance ensures that transforming the input produces a correspondingly transformed output, enabling the network to generalize across orientations without relying on explicit data augmentation.
When $\rho_{g}^{Y}$ is the identity for all $g$, equivariance reduces to invariance.

\nbf{Lifting and Regular Representation}
Under an input rotation $r_k$, a group feature map $F\colon\mathbb{R}^{2}\!\times\!G\to\mathbb{R}^{C}$ transforms as follows:
\begin{equation}
  \bigl[\rho_{r_k}(F)\bigr](x,\,g)
  = F\bigl(r_k^{-1}x,\; r_k^{-1}g\bigr),
  \label{eq:regular}
\end{equation}
concretely, a cyclic permutation by $k$ along the group axis composed with a rotation by $r_k^{-1}$.
Following this representation, the planar input $\mathbf{x}_{0}\!\in\!\mathbb{R}^{c\times h\times w}$ is lifted to a group feature map as follows:
\begin{equation}
  \mathbf{y}(g) = \sum_{i=1}^{c}
  \mathbf{x}_{0,i} \ast \mathcal{R}_{g}(\mathbf{W}_{i}),
  \quad g \in \{0,1,\dots,N{-}1\},
  \label{eq:lift}
\end{equation}
where $\mathbf{W}\!\in\!\mathbb{R}^{c' \times c \times s \times s}$ is a learnable weight matrix with kernel size $s$ and $\mathcal{R}_{g}$ denotes spatial rotation by $2\pi g/N$.
For subsequent layers, group-to-group convolutions take the form:
\begin{equation}
  \mathbf{y}'(g) = \sum_{i=1}^{c} \sum_{g'=0}^{N-1}
  \mathbf{y}_{i}(g') \ast \mathcal{R}_{g}\!\bigl(\mathbf{W}_{i,\,g^{-1}g'}\bigr),
  \quad g \in \{0,1,\dots,N{-}1\},
  \label{eq:gconv}
\end{equation}
where the relative index $g^{-1}g'$ induces a cyclic shift along the group axis.
Successive group convolutions retain this transformation law.
Equation~\eqref{eq:gconv} costs $\mathcal{O}(c^{2}Ns^{2})$, which is prohibitive for real-time detection. We therefore decompose it into two stages that jointly preserve Eq.~\eqref{eq:regular}.

\textit{Stage\;1: Depthwise rotated spatial filtering.}
Each channel $i$ is convolved with a single spatial kernel $\mathbf{W}^{\mathrm{dw}}_{i}\!\in\!\mathbb{R}^{s\times s}$, rotated per group element:
\begin{equation}
  \tilde{F}_{i}(x,g)
  = F_{i}(x,g) \ast \mathcal{R}_{g}\!\bigl(\mathbf{W}^{\mathrm{dw}}_{i}\bigr),
  \label{eq:dw}
\end{equation}
preserving the cyclic-shift law of Eq.~\eqref{eq:regular} channel-wise at $\mathcal{O}(cNs^{2})$ cost.

\textit{Stage\;2: Pointwise group mixing.}
A $1\!\times\!1$ convolution mixes channels and group indices via a weight $\mathbf{W}^{\mathrm{pw}}\!\in\!\mathbb{R}^{c'\times c\times N}$ whose group axis is cyclically shifted:
\begin{equation}
  F'_{j}(x,g)
  = \sum_{i=1}^{c}\sum_{g'=0}^{N-1}
    W^{\mathrm{pw}}_{j,i,\,g^{-1}g'}\;\tilde{F}_{i}(x,g'),
  \label{eq:pw}
\end{equation}
at $\mathcal{O}(c^{2}N)$ per spatial location.
The cyclic indexing $g^{-1}g'$ guarantees that the composition of Eqs.~\eqref{eq:dw}--\eqref{eq:pw} satisfies Eq.~\eqref{eq:regular} end-to-end with cost $\mathcal{O}(cNs^{2}+c^{2}N)$.
\subsection{Spectral Implicit Warp}
\label{sec:speiw}

While the lifting layer endows features with rotational structure, the spectral dimension remains discretized into independent channels.
SpeIW addresses this by predicting a monotone warp index $\phi$ for each spectral channel and resampling via differentiable interpolation (Fig.~\ref{fig:network} (a)), enabling adaptive spectral modeling while preserving equivariance.

\nbf{Coordinate-Conditioned Factorization}
Let $F\in\mathbb{R}^{C\times|G|\times H\times W}$ denote a group feature map with $C$ spectral channels and $|G|$ orientation slots.
Each spectral channel is assigned a normalized spectral coordinate $\lambda_c\in[-1,1]$, lifted via Fourier positional encoding $\psi(\lambda)=\bigl[\lambda,\{\sin(2^{i}\pi\lambda),\cos(2^{i}\pi\lambda)\}_{i=0}^{L-1}\bigr]\in\mathbb{R}^{2L+1}$ to mitigate the frequency bias of coordinate-based networks~\cite{tancik2020fourierfeatures,rahaman2019spectralbias}.
We decompose the raw warp field into a spatial latent code $\mathbf{z}(x,g)\!\in\!\mathbb{R}^{R}$, derived from $F$ via a learned equivariant projection, and a spectral basis decoded by a coordinate-conditioned implicit function $f_{\theta}\colon\mathbb{R}^{2L+1}\!\to\!\mathbb{R}^{R}$:
\begin{equation}
  \delta_{c}(x,g)
  = \bigl\langle\,
      \mathbf{z}(x,g),\;\;
      f_{\theta}\!\bigl(\psi(\lambda_{c})\bigr)
    \,\bigr\rangle,
    \quad c=1,\dots,C.
  \label{eq:lowrank}
\end{equation}
Because $f_{\theta}$ is shared across all $C$ spectral channels and takes continuous spectral coordinates as input, it parameterizes a spectral basis that can be queried at arbitrary coordinates, with parameter cost that does not scale with $C$.

\nbf{Order-Preserving Mapping}
Directly using $\delta_c(x,g)$ as resampling indices may cause index crossing (i.e., $\phi_c(x,g)\ge \phi_{c+1}(x,g)$), resulting in non-monotonic spectral mappings and invalid resampling.
We therefore enforce monotonically ordered indices $\{\phi_1(x,g)<\phi_2(x,g)<\cdots<\phi_C(x,g)\}$ to preserve the channel order.
We obtain nonnegative increments via softplus, and enforce strict ordering by cumulative summation followed by normalization:
\begin{subequations}\label{eq:monotone}
\begin{gather}
  \Delta_{c}(x,g) \triangleq \sigma_{+}\!\bigl(\delta_{c}(x,g)\bigr)+\varepsilon_s, \label{eq:monotone_a}\\
  \phi_{c}(x,g) = (C{-}1)\cdot\frac{\sum_{j=2}^{c}\Delta_{j}(x,g)}{\sum_{j=2}^{C}\Delta_{j}(x,g)+\varepsilon_d}, \label{eq:monotone_b}
\end{gather}
\end{subequations}
where $\sigma_{+}$ denotes softplus and $\varepsilon_s,\varepsilon_d>0$ are numerical constants used by the implementation. For finite real-arithmetic values, each $\Delta_c$ is positive and the indices are strictly ordered; the denominator offset places the last index slightly below $C-1$. Floating-point rounding can collapse extremely small adjacent increments, so the implementation-level guarantee is an ordered, numerically non-decreasing map rather than bitwise strict inequality for every input. Componentwise finite replacement and clipping keep the sampled indices in $[0,C-1]$.

\nbf{Continuous Spectral Resampling}
Given the monotonically ordered indices $\{\phi_c(x,g)\}_{c=1}^C$, $\phi_c(x,g)\in[0,C{-}1]$, we obtain the warped feature via differentiable linear interpolation along the spectral dimension:
\begin{equation}
  \hat{F}_{c}(x,g)
  = \mathcal{I}\!\bigl[\,F(x,g,\cdot)\,;\;\phi_{c}(x,g)\,\bigr],
  \label{eq:spectral_field}
\end{equation}
where $\mathcal{I}[\,\cdot\,;\,t\,]$ denotes a differentiable linear interpolation operator evaluated at continuous index $t$.
Since $\mathbf{z}(x,g)$ is rotation-equivariant and $f_{\theta}$ is shared across orientations, the predicted indices $\phi_{c}(x,g)$ are equivariant as well.
Since the interpolation operates only along the spectral dimension and is applied independently for each $(x,g)$, it commutes with the group action.

\subsection{Rotation-Equivariant Consistency Weighting}
\label{sec:recog}

Multi-scale fusion is essential for scale robustness, but repeated cross-level aggregation can wash out discriminative spectral evidence. ReCoW addresses this issue (Fig.~\ref{fig:network} (b)) with complementary spectral and spatial routing followed by agreement-weighted residual modulation. The released operator refines the original group feature; it is not a scalar convex interpolation between two branches.

\nbf{Dual-Branch Prototype Routing}
Both branches operate on local windows of size $W\!\times\!W$.
A rotation-invariant descriptor $\bar{F}(x)=\frac{1}{|G|}\sum_g F(x,g)$ is first obtained by group averaging.
Within each window, we form $K\!=\!q^{2}$ prototype centers $\{\mathbf{c}_k\}_{k=1}^{K}$ by adaptive average pooling of $\bar{F}$ to a $q\!\times\!q$ spatial grid.
The grid preserves local spatial layout and provides region-specific centers, where $x$ indexes locations within the window.

The two branches adopt complementary routing rules. Spectral signatures often exhibit continuous mixing; thus, we employ cosine-based soft assignment:
\begin{subequations}\label{eq:soft_assign}
\begin{gather}
  s^{\mathrm{spec}}_{k}(x) = \tau\cdot\cos\!\bigl(\mathbf{c}_k,\,\bar{F}(x)\bigr), \label{eq:soft_assign_a}\\
  p_k(x) = \frac{\exp\!\bigl(s^{\mathrm{spec}}_{k}(x)\bigr)}{\sum_{k'}\exp\!\bigl(s^{\mathrm{spec}}_{k'}(x)\bigr)}, \label{eq:soft_assign_b}
\end{gather}
\end{subequations}
where $\tau$ is a learnable temperature.
The prototype reconstruction is
\begin{equation}
  \mathbf{r}_{\mathrm{spec}}(x) = \sum_{k=1}^{K} p_k(x)\,\mathbf{c}_k .
  \label{eq:spec_agg}
\end{equation}
It is scaled by the probability-weighted prototype similarity and its cosine agreement with $\bar F(x)$, each mapped to $[0,1]$. The result is broadcast across $g$ and passed through an equivariant pointwise post-projection to obtain $S(x,g)$.

Spatial grouping benefits from crisp region assignments; accordingly, we perform hard routing on a rotation-invariant spatial score map.
We compute a scalar score $s(x)=\frac{1}{C|G|}\sum_{c,g}F_c(x,g)$ and obtain scalar prototype centers $\{\mu_k\}_{k=1}^{K}$ by adaptive average pooling onto the same grid. Hard assignment uses a Gaussian affinity between normalized scores:
\begin{subequations}\label{eq:hard_assign}
\begin{gather}
  a_k(x) = \exp\!\bigl(-\gamma\,\|\hat\mu_k-\hat s(x)\|_2^2\bigr), \label{eq:hard_assign_a}\\
  k^{*}(x) = \argmax_{k} a_k(x), \label{eq:hard_assign_b}
\end{gather}
\end{subequations}
where $\gamma$ is a learnable inverse bandwidth and hats denote normalized scalar descriptors. Given $k^{*}(x)$, the spatial branch computes a per-group cluster mean:
\begin{equation}
  \mathbf{r}_{\mathrm{spat}}(x,g)
  = \frac{1}{\lvert \Omega_{k^*(x)} \rvert}\sum_{x'\in \Omega_{k^*(x)}} F(x',g),
  \label{eq:spat_agg}
\end{equation}
where $\Omega_{k}$ denotes the set of locations assigned to prototype $k$ in the window.
The selected feature is weighted by a sigmoid of the top-one/top-two affinity margin and transformed by an equivariant depthwise post-projection to obtain $T(x,g)$.

\nbf{Agreement-Weighted Residual Modulation}
Let $\bar S$ and $\bar T$ be the group averages of the routed features. Their invariant agreement is
\begin{equation}
  A(x)=\frac{1+\cos\!\bigl(\bar S(x),\bar T(x)\bigr)}{2}\in[0,1].
  \label{eq:agreement}
\end{equation}
The routed tensors jointly form an equivariant gate, and ReCoW returns
\begin{equation}
  F'(x,g)=F(x,g)+F(x,g)\odot\Bigl[A(x)\,\sigma\!\bigl(S(x,g)+T(x,g)\bigr)\Bigr],
  \label{eq:recow_out}
\end{equation}
where $\odot$ denotes componentwise multiplication. On a compatible square window lattice, group averaging makes $A$ invariant in the orientation index, while $S$, $T$, and the sigmoid gate carry the regular representation. Pointwise sums and products preserve the group action. The hard-routing statement assumes a unique winning prototype; exact affinity ties are resolved by implementation order.

\subsection{Oriented-Aware Head Network}
\label{sec:head}

Detection requires category-score fields that are invariant along the orientation axis while remaining spatially equivariant, whereas box and angle predictions must transform consistently with the input.
Let $F_{\!\mathrm{n}}\!\in\!\mathbb{R}^{C\times|G|\times H\times W}$ denote the group feature from the neck and $\rho_g$ the group action defined in Eq.~\eqref{eq:regular}.
We design the three branches of the head to satisfy (Fig.~\ref{fig:network}\,(c)):
\begin{subequations}\label{eq:head_equiv}
\begin{gather}
  f_{\mathrm{cls}}(\rho_g\,F_{\!\mathrm{n}})(x) \;=\; f_{\mathrm{cls}}(F_{\!\mathrm{n}})(g^{-1}x),
  \label{eq:cls_inv}\\[2pt]
  f_{\mathrm{box}}(\rho_g\,F_{\!\mathrm{n}}) \;=\; \rho_g\,f_{\mathrm{box}}(F_{\!\mathrm{n}}),
  \label{eq:box_equiv}\\[2pt]
  f_{\mathrm{ang}}(\rho_g\,F_{\!\mathrm{n}})(x) \;=\; \operatorname{wrap}_{\pi}\!\left(
    f_{\mathrm{ang}}(F_{\!\mathrm{n}})(g^{-1}x)
    + \tfrac{2\pi}{|G|}\,\operatorname{idx}(g)\right),
  \label{eq:ang_equiv}
\end{gather}
\end{subequations}
where $\operatorname{idx}(g)$ returns the integer index of $g$ in $G$, and $\operatorname{wrap}_{\pi}$ denotes the $\pi$-periodic OBB angle normalization.

All three branches are built from the depthwise-pointwise equivariant convolutions of Eqs.~\eqref{eq:dw}--\eqref{eq:pw} with shared parameters across orientations, keeping the head compact.
The classification branch achieves invariance via group mean pooling at the readout stage:
\begin{equation}\label{eq:group_pool}
  \mathbf{z}_{\mathrm{inv}} \;=\;
    \frac{1}{|G|}\sum_{g\in G}\mathbf{z}_g\,,
\end{equation}
where $\mathbf{z}_g$ is the feature response at orientation $g$; since group pooling removes only the orientation index while preserving the transformed spatial coordinate, Eq.~\eqref{eq:cls_inv} follows directly.
The box branch keeps group-indexed features and uses a cyclic-tied DFL readout whose weights depend on relative group offsets, while the angle branch predicts group-indexed residual candidates and decodes them by soft circular aggregation.
The box readout preserves the side-permutation law. The angle shift law holds whenever the weighted circular resultant is nonzero; exact or near-zero resultants are ambiguous and are treated as a numerical degeneracy rather than covered by an unconditional theorem.
\section{Experiments}

We evaluate FressDet on five multispectral object detection benchmarks: MODA \cite{han2025moda}, HOD3K \cite{hod3k}, and DroneVehicle \cite{dronevehicle} in the main paper, with LLVIP~\cite{jia2021llvip} and VEDAI~\cite{razakarivony2016vedai} reported in the supplementary material.

\subsection{Implementation Details}
\label{sec:impl}

The backbone employs four stages, each equipped with Fourier positional encoding of $L\!=\!8$ frequencies and an implicit MLP of depth 5 with ReLU activations. The neck applies ReCoW at each fusion level with local window size $W\!=\!4$ and $K\!=\!16$ prototypes arranged on a $4\!\times\!4$ spatial grid. The rotation group is constructed on $C_4$ in all experiments unless otherwise specified.

FressDet is trained from scratch without data augmentation, as the built-in rotation equivariance provides sufficient geometric robustness. The framework is built on Ultralytics YOLO~\cite{yolov8} with default settings, expanding the first convolution to match the number of MSI bands. All experiments are conducted in PyTorch on two NVIDIA RTX 3090 GPUs.

\begin{table*}[t]
\centering
\scriptsize
\setlength{\tabcolsep}{0.7pt}
\caption{Comparison with other methods on MODA. R-50: ResNet-50, ReR-50: ReResNet-50, ReR-N-50: ReResNet-N-50, CSP-D: CSPDarknet, CSP-N: CSPNeXt, RE-CSP-N: RE-CSPNeXt. \#P (M): trainable parameters. \#F (G): inference FLOPs.}
\label{tab:moda}
\begin{tabular}{l|c|cccccccc|ccc|cc}
\toprule
Methods & Backbone & Car & Bus & Van & Awi. & Tru. & Tri. & Bike & Ped. &
mAP$_{50}$ & mAP$_{75}$ & mAP & \textbf{\#F}  & \textbf{\#P} \\
\midrule
\multicolumn{15}{c}{\textbf{one-stage methods}} \\
R~Retina.\cite{retinanet}       & R-50    & 90.2 & 81.9 & 72.4 & 59.0 & 45.5 & 30.8 & 21.7 & 25.4 & 53.4 & 37.3 & 34.1 & 233.5 & 36.5 \\
GWD~\cite{yang2021rethinking}      & R-50    & 90.4 & 79.9 & 70.7 & 59.7 & 49.3 & 30.8 & 26.7 & 29.7 & 54.7 & 37.8 & 34.6 & 233.5 & 36.5 \\
RepPts-R~\cite{reppoints}          & R-50    & 89.9 & 70.2 & 67.0 & 66.1 & 50.2 & 41.5 & 33.3 & 40.5 & 57.3 & 34.9 & 33.0 & 213.6 & 36.9 \\
R3Det~\cite{yang2021r3det}         & R-50    & 90.4 & 88.7 & 74.0 & 63.8 & 57.8 & 45.1 & 29.4 & 32.2 & 60.2 & 36.7 & 34.8 & 362.2 & 42.0 \\
R3Det-KLD~\cite{yang2021learning}  & R-50    & 90.4 & 88.7 & 74.8 & 66.6 & 60.9 & 42.6 & 29.8 & 35.5 & 61.2 & 39.4 & 36.8 & 306.7 & 39.6 \\
YOLOv8n~\cite{yolov8}              & CSP-D   & 95.8 & 85.2 & 63.7 & 67.4 & 33.6 & 35.9 & 49.6 & 65.1 & 62.0 & 48.0 & 43.3 & \textbf{12.6}  & 3.2  \\
S$^2$ANet~\cite{han2021align}      & R-50    & 90.4 & 88.7 & 74.4 & 69.1 & 62.7 & 48.9 & 30.1 & 40.7 & 63.1 & 38.7 & 36.7 & 216.5 & 38.8 \\
S2ADet~\cite{hod3k}                & R-50    & 90.3 & 86.6 & 72.1 & 71.3 & 57.2 & 54.1 & 35.0 & 40.8 & 63.5 & 41.1 & 38.9 & 406.0 & 65.2 \\
FCOS-R~\cite{fcos}         & R-50    & 90.3 & 86.3 & 75.2 & 71.2 & 59.0 & 53.8 & 34.7 & 37.4 & 63.5 & 41.8 & 39.1 & 226.9 & 32.2 \\
YOLO26m~\cite{sapkota2025yolo26}   & CSP-D   & 95.5 & 85.0 & 67.3 & 74.4 & 35.7 & 33.8 & 56.9 & 66.6 & 64.4 & 51.2 & 46.5 & 167.2 & 31.8 \\
RTMDet~\cite{lyu2022rtmdet}        & CSP-N   & 90.0 & 88.4 & 75.3 & 71.1 & 62.8 & 52.5 & 42.9 & 41.0 & 65.5 & 41.5 & 40.1 & 258.5 & 52.3 \\
YOLOv8m~\cite{yolov8}              & CSP-D   & 96.2 & 87.1 & 64.7 & 73.3 & 40.4 & 36.3 & 57.0 & 69.4 & 65.6 & 53.5 & 47.9 & 112.0 & 26.4 \\
YOLOv12m~\cite{tian2025yolov12}    & CSP-D   & 96.3 & 85.3 & 68.8 & 74.3 & 40.3 & 36.3 & 57.4 & 69.5 & 66.0 & 52.4 & 46.8 & 115.0 & 24.7 \\
CFA~\cite{cfa}                     & R-50    & 90.4 & 89.0 & 76.7 & 69.7 & 64.4 & 55.1 & 43.1 & 41.5 & 66.2 & 43.2 & 40.6 & 213.6 & 36.9 \\
O-RepPts~\cite{oriented-reppoints} & R-50    & 90.5 & 89.2 & 77.7 & 71.2 & 66.2 & 53.1 & 43.0 & 41.1 & 66.5 & 44.1 & 40.9 & 213.6 & 36.9 \\
YOLO11m~\cite{yolo11_ultralytics}    & CSP-D   & 96.4 & 88.0 & 68.2 & 75.7 & 44.4 & 38.0 & 58.1 & 68.7 & 67.2 & 54.6 & 48.9 & 126.2 & 25.1 \\
DCFL~\cite{Xu_2023_CVPR}             & ReR-50  & 90.6 & 89.7 & 77.5 & 71.3 & 67.2 & 55.5 & 44.2 & 44.0 & 67.5 & 44.8 & 41.5 & 241.0 & 37.0 \\
MessDet~\cite{wu2025measuring} & RE-CSP-N& 90.0 & 89.1 & 75.4 & 71.4 & 70.2 & 54.9 & 48.4 & 45.5 & 68.1 & 42.4 & 40.1 & 616.0 & 15.1 \\
OSSDet~\cite{han2025moda}          & R-50    & 90.5 & 89.9 & \textbf{79.2} & 72.7 & \textbf{69.7} & \textbf{58.8} & 45.3 & 45.7 & 69.0 & 45.9 & 42.7 & 263.1 & 36.5 \\
\midrule
\multicolumn{15}{c}{\textbf{two-stage methods}} \\
GV~\cite{Gliding-vertex}           & R-50    & 90.3 & 89.1 & 73.8 & 69.5 & 66.0 & 46.7 & 41.4 & 22.6 & 62.4 & 35.7 & 34.7 & 230.8 & 41.4 \\
RoI-T~\cite{roi_trans}             & R-50    & 90.5 & 89.3 & 75.2 & 73.3 & 68.7 & 51.8 & 44.5 & 30.0 & 65.4 & 43.4 & 40.7 & 244.7 & 55.3 \\
O-RCNN~\cite{xie2021oriented}      & R-50    & 90.5 & 89.7 & 74.6 & 72.5 & 66.6 & 54.7 & 45.1 & 30.4 & 65.5 & 44.0 & 40.9 & 244.7 & 55.3 \\
LSKNet~\cite{Li_2024_IJCV}         & LSKNet-S& 90.4 & 88.5 & 74.5 & 72.5 & 66.0 & 52.5 & 41.0 & 28.2 & 64.2 & 42.8 & 39.4 & 221.9 & 32.6 \\
StripRCNN~\cite{yuan2025strip}     & R-50    & 90.5 & 89.0 & 76.1 & 73.6 & 67.1 & 56.5 & 45.2 & 30.7 & 66.1 & 44.0 & 41.0 & 231.8 & 45.2 \\
ReDet~\cite{redet}                 & ReR-N-50& 90.6 & 89.6 & 76.8 & 72.9 & 69.6 & 52.4 & 41.5 & 41.0 & 66.8 & 44.1 & 41.3 & 245.2 & 31.6 \\
\hline\noalign{\smallskip}
\multicolumn{15}{c}{\textbf{DETR based methods}} \\
AO$^2$-DETR~\cite{dai2022ao2}       & R-50    & 89.4 & 88.1 & 76.0 & 69.5 & 65.7 & 54.8 & 41.9 & 41.0 & 65.8 & 43.0 & 40.0 & 240.0 & 42.0 \\
ARS-DETR~\cite{zeng2024ars}        & Swin-T  & 90.5 & 89.2 & 77.2 & 70.6 & 64.8 & 53.6 & 43.2 & 39.7 & 66.1 & 43.4 & 40.5 & 254.0 & 44.2 \\
RHINO~\cite{Lee_2025_WACV}         & Swin-T  & 90.7 & 89.1 & 77.3 & 70.8 & 65.6 & 54.0 & 42.9 & 41.0 & 66.4 & 44.0 & 40.8 & 319.5 & 50.8 \\
O-DETR~\cite{zhao2024pointaxis}    & Swin-T  & 90.8 & 89.3 & 78.8 & 72.8 & 68.0 & 56.0 & 46.6 & 40.1 & 67.8 & 45.1 & 41.9 & 492.0 & 57.6 \\
\midrule
\rowcolor{LightPurple}
FressDet (Ours) & SpeIWNet & \textbf{97.7} & \textbf{90.3} & 75.4 & \textbf{77.8} & 59.5 & 49.4 & \textbf{62.0} & \textbf{72.5} & \textbf{73.1} & \textbf{60.2} & \textbf{54.3} & 31.7 & \textbf{2.3} \\
\bottomrule
\end{tabular}
\end{table*}
\subsection{Comparison with State-of-the-Art Methods}
\label{sec:sota}
\nbf{Results on MODA}
Tab.~\ref{tab:moda} reports results on MODA. FressDet outperforms OSSDet by 4.1\% mAP$_{50}$, 14.3\% mAP$_{75}$, and 11.6\% mAP, with only 6.3\% of OSSDet's parameters and 12.0\% of its FLOPs, and achieves large gains on small classes (bike +16.7\%, pedestrian +26.8\%). Fig.~\ref{fig:vis_moda} demonstrates fewer misses and false positives under clutter and low visibility. Fig.~\ref{fig:key_com} (a) further verifies component effectiveness: SpeIW yields more localized activations through continuous spectral resampling, while ReCoW uses routed cue agreement to modulate target responses and attenuate clutter. Fig.~\ref{fig:key_com} (b) shows that continuous spectral modeling enlarges inter-class margins (e.g., car vs. van) and forms tighter clusters for small classes (bike, pedestrian), improving discriminability.

\nbf{Results on HOD3K}
Tab.~\ref{tab:hod3k} shows FressDet surpasses the previous best OSSDet by 0.2\% in mAP while achieving the best mAP$_{50}$ with only 6.3\% of OSSDet's parameters and 26.8\% of its FLOPs.
Fig.~\ref{fig:hod3k_vis} shows qualitative results on HOD3K. Competing methods often miss occluded or background-blended objects, whereas FressDet enhances target--background separation, yielding predictions closer to the ground truth.

\begin{figure*}[!t]
\centering
\includegraphics[width=.94\linewidth]{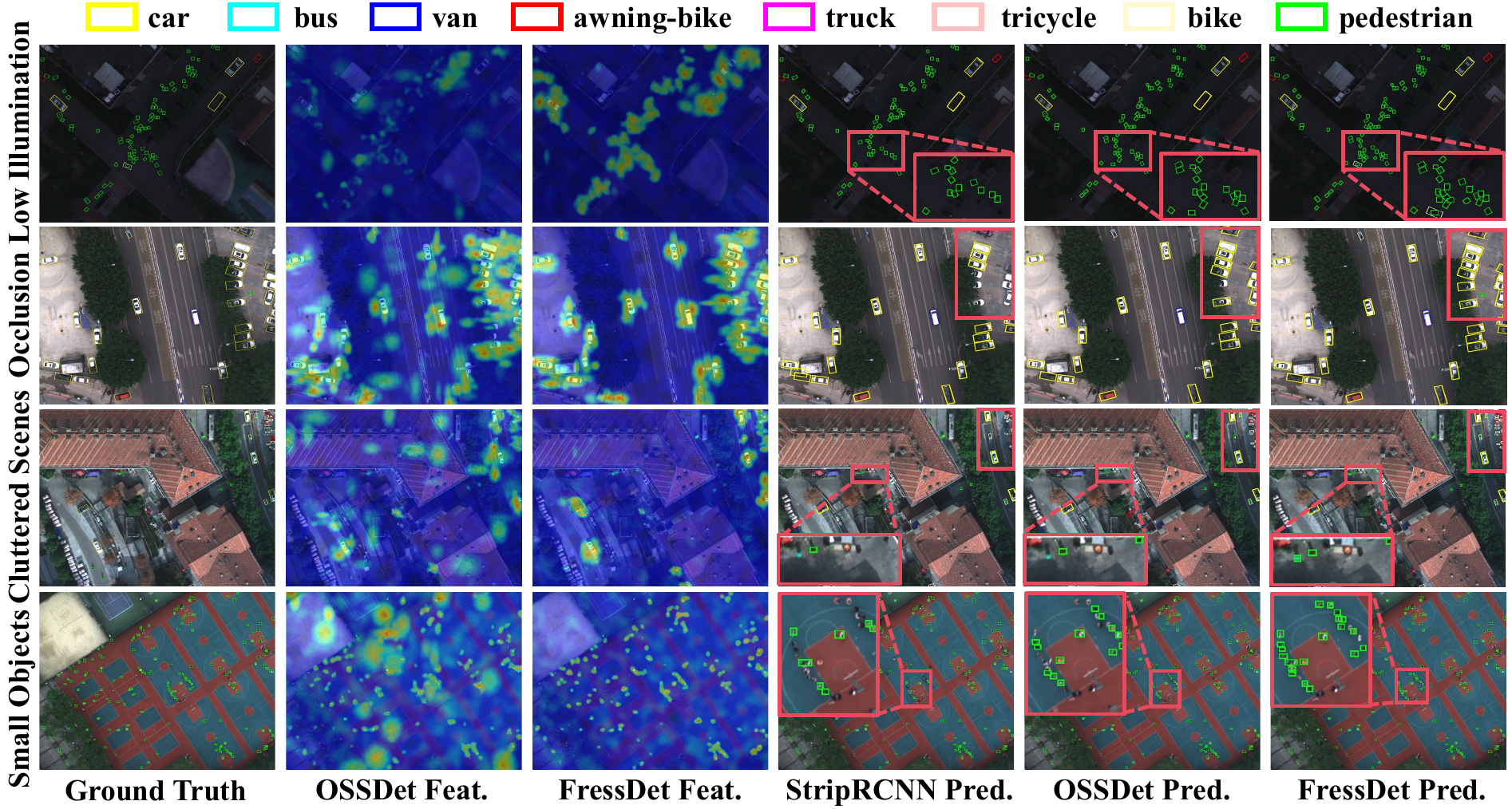}
\figcaption{Visualization comparison of detection results and feature maps obtained by different methods on MODA. Please zoom in to see details.}
\label{fig:vis_moda}

\scriptsize
\setlength{\tabcolsep}{5.2pt}
\tabcaption{Performance comparisons on HOD3K.}
\label{tab:hod3k}
\begin{tabular}{l|ccc|ccc|cc}
\toprule
Methods & People & Bike & Car & mAP$_{50}$ & mAP$_{75}$ & mAP & FLOPs & \textbf{\#P} \\
\midrule
Deformable DETR~\cite{zhu2021deformable} & 52.8 & 56.3 & 64.9 & 58.0 & 24.8 & 22.3 & 195.2G & 39.8M \\
FCOS~\cite{fcos}                         & 55.0 & 19.7 & 69.8 & 48.2 & 26.3 & 23.7 & 196.8G & 31.8M \\
YOLOF~\cite{chen2021you}              & 60.8 & 67.8 & 67.6 & 65.4 & 31.7 & 28.6 & 98.2G  & 42.1M \\
CO-DETR~\cite{zong2023detrs}            & 70.2 & 75.3 & 88.2 & 77.9 & 57.3 & 51.7 & 200.7G & 64.5M \\
RetinaNet~\cite{retinanet}          & 85.8 & 95.0 & 92.8 & 91.2 & 59.4 & 53.3 & 205.7G & 36.2M \\
YOLOv5~\cite{yolov5}             & 79.3 & 94.0 & 91.2 & 88.1 & 59.8 & 54.4 & 48.3G  & 20.9M \\
DINO~\cite{dino}               & 82.5 & 92.3 & 92.5 & 89.1 & 61.6 & 56.3 & 114.7G & 47.6M \\
Fovea~\cite{foveabox}              & 88.1 & 95.8 & 92.7 & 92.2 & 62.2 & 56.8 & 123.9G & 38.0M \\
Faster RCNN~\cite{fasterrcnn}        & 81.8 & 94.5 & 91.7 & 89.4 & 62.7 & 56.9 & 206.7G & 41.1M \\
ATSS~\cite{atss}               & 87.2 & 94.0 & 92.5 & 91.2 & 64.3 & 57.4 & 110.6G & 32.2M \\
TOOD~\cite{tood}               & 88.1 & 90.1 & 92.7 & 90.3 & 65.6 & 57.8 & 108.5G & 32.1M \\
MethaneMapper~\cite{kumar2023methanemapper}      & 80.2 & 95.1 & 90.9 & 88.7 & 63.8 & 57.9 & 267.6G & 80.3M \\
VFNet~\cite{vfnet}              & 88.5 & 96.2 & 92.2 & 92.3 & 66.5 & 59.0 & 104.5G & 32.8M \\
S2ADet~\cite{hod3k}             & 87.2 & \best{97.7} & 95.3 & 93.4 & 66.2 & 59.8 & 169.2G & 48.6M \\
OSSDet~\cite{han2025moda} & 88.3 & 96.5 & \best{95.4} & 93.4 & 68.8 & 60.9 & 131.2G & 36.6M \\
\midrule
\rowcolor{LightPurple}
FressDet (Ours) & \best{88.9} & 97.1 & 95.3 & \best{93.8} & \best{69.1} & \best{61.1} & \best{35.2G} & \best{2.3M} \\
\bottomrule
\end{tabular}
\end{figure*}

\nbf{Results on DroneVehicle}
As shown in Tab.~\ref{tab:dv}, our method surpasses the best~\cite{zhu2025wavemamba} by 1.1\% mAP and 0.5\% mAP$_{50}$ under identical settings.
FressDet achieves the best Car and Bus mAP$_{50}$, leading by 1.4\% and 1.3\% over the best competing methods in each category.
FressDet benefits from the representational continuity that naturally emerges along the spectral dimension.

\begin{figure*}[!t]
    \centering
    \includegraphics[width=\linewidth]{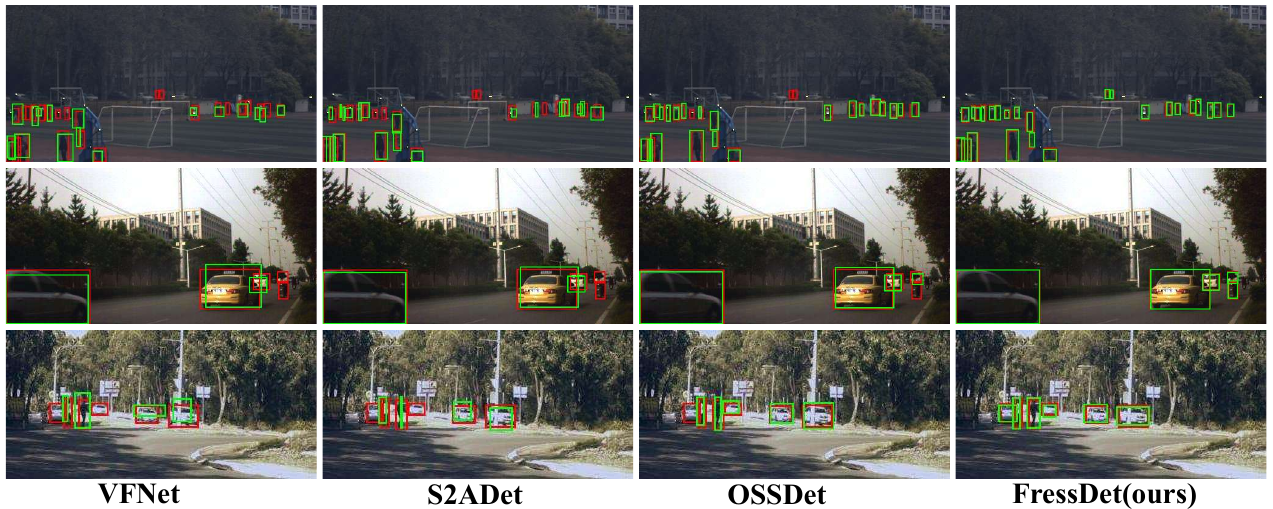}
    \caption{Detection results on HOD3K. Red boxes indicate the ground truth, while green boxes denote the predictions.}
    \label{fig:hod3k_vis}
\end{figure*}

\begin{table*}[!t]
\centering
\scriptsize
\setlength{\tabcolsep}{7.2pt}
\caption{Performance comparisons on DroneVehicle.}
\label{tab:dv}
\begin{tabular}{l|ccccc|cc}
\toprule
Methods & Car & Bus & Truck & Freight-car & Van & mAP$_{50}$ & mAP \\
\midrule
$C^2$Former-$S^2$ANet~\cite{yuan2024c2former} & 90.2 & 89.8 & 68.3 & 64.4 & 58.5 & 74.2 & 47.3 \\
YOLOv8l-RGB~\cite{yolov8}              & 92.5 & 91.7 & 68.8 & 47.8 & 50.2 & 70.2 & 48.6 \\
YOLOv8l-IR~\cite{yolov8}               & 93.4 & 91.9 & 69.3 & 53.7 & 51.1 & 71.9 & 49.1 \\
SLBAF-Net~\cite{cheng2023slbaf}            & 90.2 & 89.9 & 76.0 & 68.2 & 59.9 & 76.8 & 49.5 \\
CSOM-ODAF~\cite{chen2024weakly}     & 90.1 & 89.8 & 75.6 & 68.2 & 61.8 & 77.1 & 50.1 \\
Multimodal DINO~\cite{sun2024freqgate} & 89.5 & 88.8 & 75.4 & 54.3 & 54.3 & 72.5 & 50.3 \\
CMA~\cite{jiang2024m2fnet}            & 95.8 & 93.1 & 75.9 & 59.8 & 59.4 & 76.8 & 50.4 \\
CRSIOD~\cite{wang2024cross}               & 95.6 & 92.2 & 71.7 & 50.5 & 55.8 & 73.2 & 50.8 \\
YOLOFIV~\cite{wang2024yolofiv}              & 95.9 & 91.6 & 64.2 & 34.6 & 37.3 & 64.7 & 53.1 \\
GLFNet~\cite{kang2024glfnet}               & 90.3 & 88.0 & 72.7 & 53.6 & 52.6 & 71.4 & 54.8 \\
Dual-YOLO~\cite{bao2023dual}           & 95.9 & 91.6 & 69.7 & 55.9 & 46.6 & 71.9 & 55.2 \\
DMM~\cite{zhou2024dmm}  & 90.4 & 88.7 & 77.8 & 63.0 & \best{66.0} & 77.2 & 55.8 \\
CAFN-IA~\cite{xu2024cross}              & 89.1 & 90.8 & 62.0 & 57.3 & 47.1 & 69.3 & 56.1 \\
IV-YOLO~\cite{tian2024ivyolo} & 97.2 & 94.3 & 65.4 & 63.1 & 53.0 & 74.6 & 56.8 \\
WaveMamba~\cite{zhu2025wavemamba}           & 95.0 & 90.6 & \best{80.4} & \best{68.5} & 64.5 & 79.8 & 60.5 \\
\midrule
\rowcolor{LightPurple}
FressDet (Ours) & \best{98.6} & \best{95.6} & 79.0 & 65.9 & 62.4 & \best{80.3} & \best{61.6} \\
\bottomrule
\end{tabular}
\end{table*}
\subsection{Ablation Studies and Analysis}
\label{sec:ablation}
All ablations are conducted on MODA under identical settings unless noted.

\nbf{Effectiveness of FressDet's Key Components}
As shown in Tab.~\ref{tab:ablation_key}, introducing rotation-equivariant group convolutions alone yields a +4.4\% mAP gain over the baseline, confirming that geometric priors are as critical as spectral modeling for multispectral detection. Each subsequent module provides improvements, and the full model achieves the best result across all metrics.
Among the three proposed modules, SpeIW contributes the most, underscoring the value of continuous spectral modeling in multispectral object detection.

\nbf{Ablation on Spectral Operator}
We ablate the spectral operator by replacing SpeIW with rotation-equivariant DCN variants~\cite{dai2017deformable,Zhu_2019_CVPR,wang2023internimage,xiong2024efficient} (implementation details in the supplementary material). Tab.~\ref{tab:speiw_compare} shows that the plain C$_4$-CNN performs worst, underscoring the value of adaptive spectral resampling. SpeIW surpasses all DCN-based alternatives, outperforming the closest DCN variant by 1.1 mAP, indicating that order-preserving continuous resampling provides a stronger inductive bias than discrete deformable offsets.

\nbf{Transferability Across Detectors}
To evaluate SpeIWNet’s generalization capability, we replace the backbones of five detectors while keeping other components and training unchanged.
Fig.~\ref{fig:generalization} shows consistent performance gains across single- and two-stage architectures, along with reduced model size, demonstrating SpeIWNet’s effectiveness as a strong plug-and-play feature extractor across heterogeneous detectors.

\begin{figure*}[!t]
    \centering
    \includegraphics[width=\linewidth]{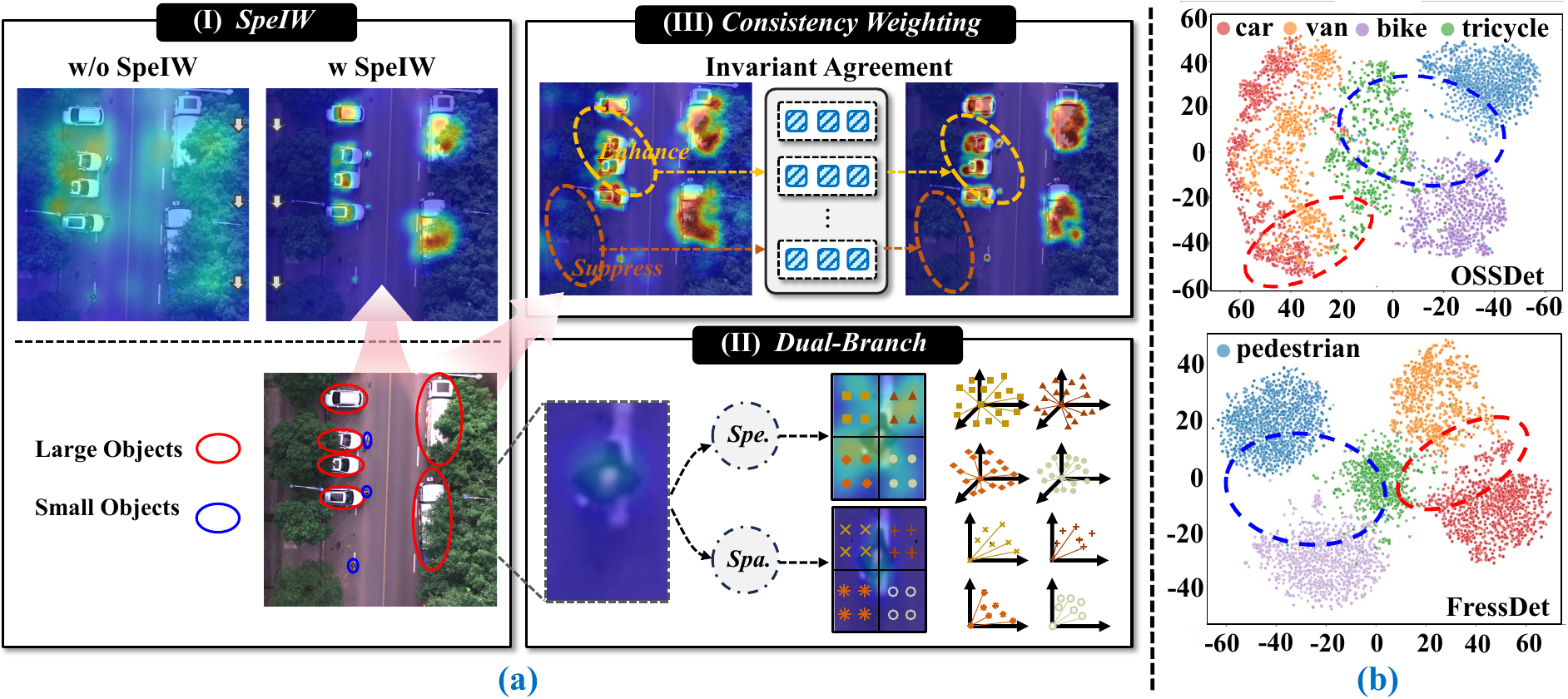}
    \caption{(a) Effectiveness visualization of key components. (b) 2D t-SNE~\cite{van2008visualizing} visualization of feature distributions on MODA.}
    \label{fig:key_com}
\end{figure*}

\begin{table}[!t]
    \begin{minipage}[t]{0.55\textwidth}
    \centering 
    \scriptsize
        \tabcaption{Ablation studies on the key components of FressDet. SIW: SpeIW, RCW: ReCoW, OH: Oriented-Aware Head.}
        \setlength{\tabcolsep}{2.3pt}{}
        \begin{tabular}{c | c c c | c c c} 
        \toprule[1.2pt]
        \textbf{Equiv.} & \textbf{SIW} & \textbf{RCW} & \textbf{OH} &
        \textbf{mAP$_{50}$} & \textbf{mAP$_{75}$} & \textbf{mAP} \\
        \midrule
        \xmark & \xmark & \xmark & \xmark & 62.0 & 48.0 & 43.3 \\
        \midrule
        \cmark & \xmark & \xmark & \xmark & 66.8 & 54.1 & 47.7 \\
        \cmark & \cmark & \xmark & \xmark & 69.9 & 55.8 & 50.3 \\
        \cmark & \xmark & \cmark & \xmark & 69.0 & 55.1 & 49.6 \\
        \cmark & \xmark & \xmark & \cmark & 68.6 & 55.5 & 49.3 \\
        \cmark & \cmark & \cmark & \xmark & 71.7 & 58.8 & 52.5 \\
        \cmark & \cmark & \xmark & \cmark & 71.6 & 58.4 & 52.2 \\
        \cmark & \cmark & \cmark & \cmark & \best{73.1} & \best{60.2} & \best{54.3} \\
        \bottomrule[1.2pt]
        \end{tabular}
        \label{tab:ablation_key}
    \end{minipage}
    \hfill
    \begin{minipage}[t]{0.41\textwidth}
    \renewcommand\arraystretch{1.28}
        \centering 
        \scriptsize
        \tabcaption{Comparison of spectral operators in the SpeIWNet backbone.}
        \setlength{\tabcolsep}{3.0pt}{}
        \begin{tabular}{l|ccc}
        \toprule[1.2pt]
        \textbf{Method} & \textbf{mAP$_{50}$} & \textbf{mAP$_{75}$} & \textbf{mAP} \\
        \midrule
        C$_4$-CNN              & 71.4 & 58.6 & 52.1 \\
        RE-DCN                 & 71.6 & 58.7 & 52.9 \\
        RE-DCNv2               & 72.0 & 59.3 & 53.2 \\
        RE-DCNv3               & 71.7 & 59.0 & 53.1 \\
        RE-DCNv4               & 71.9 & 59.2 & 53.0 \\
        \midrule
        \rowcolor{LightPurple}
        SpeIW  & \best{73.1} & \best{60.2} & \best{54.3} \\
        \bottomrule[1.2pt]
        \end{tabular}
        \label{tab:speiw_compare}
    \end{minipage}
\end{table}

\begin{figure}[!t]
    \begin{minipage}[t]{0.48\textwidth}
    \centering
    \includegraphics[width=\linewidth]{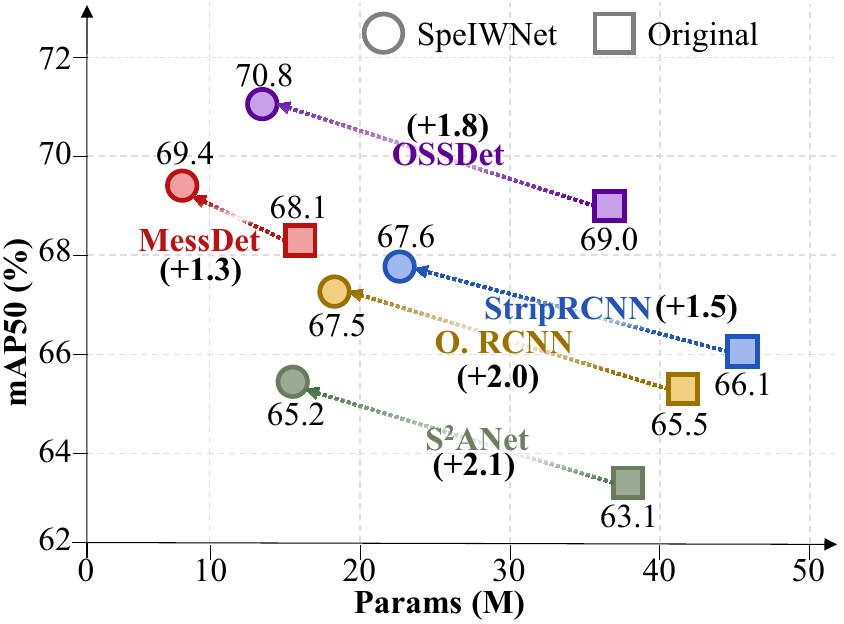}
    \figcaption{SpeIWNet backbone generalization across detectors on MODA.}
    \label{fig:generalization}
    \end{minipage}
    \hfill
    \begin{minipage}[t]{0.50\textwidth}
    \centering
    \includegraphics[width=\linewidth]{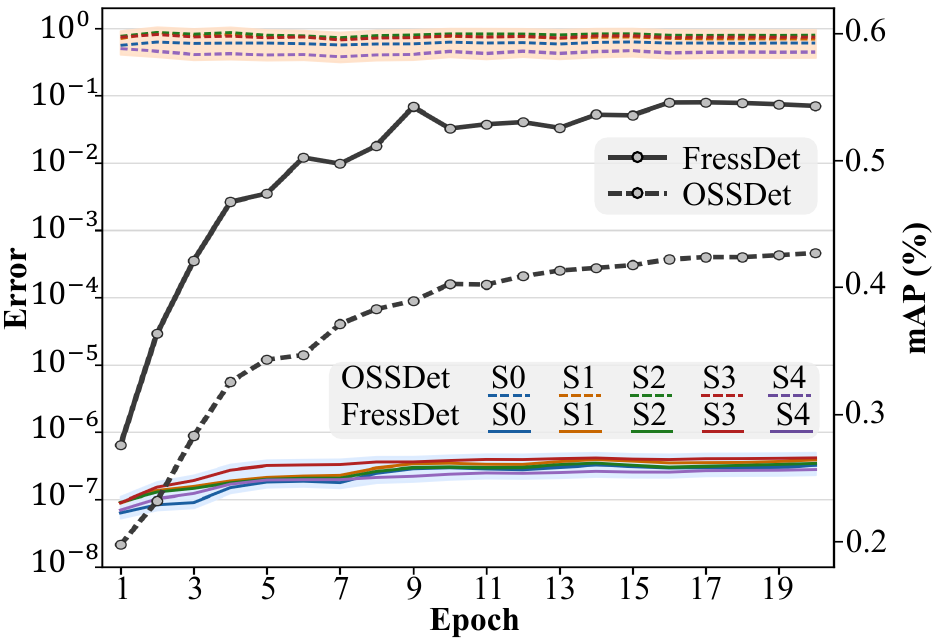}
    \figcaption{Equivariance error across backbone stages (S0: stem, S1–S4: stages 1–4).}
    \label{fig:equiv_error}
    \end{minipage}

    \begin{minipage}[t]{0.46\textwidth}
    \renewcommand\arraystretch{1.5}
        \centering 
        \scriptsize
        \tabcaption{Ablation on group order in the FressDet backbone.}
        \setlength{\tabcolsep}{0.9pt}{}
        \makebox[\linewidth][l]{%
        \begin{tabular}{c|ccc|cc}
        \toprule[1.2pt]
        \textbf{Group} & \textbf{mAP$_{50}$} & \textbf{mAP$_{75}$} & \textbf{mAP} & 
        \textbf{\#F}(G) & \textbf{\#P}(M) \\
        \midrule
        $C_2$ & 71.2 & 58.3 & 51.8 & 15.4 & 1.1 \\
        \rowcolor{LightPurple}
        $C_4$ & \best{73.1} & \best{60.2} & \best{54.3} & 31.7 & 2.3 \\
        $C_8$ & 73.2 & 60.3 & 54.3 & 65.1 & 4.8 \\

        \bottomrule[1.2pt]
        \end{tabular}%
        }
        \label{tab:group_order}
    \end{minipage}
    \hfill
    \begin{minipage}[t]{0.49\textwidth}
    \centering \scriptsize
        \tabcaption{Fourier frequency $L$ ablation in the SpeIW module.}
        \setlength{\tabcolsep}{11pt}{}
        \makebox[\linewidth][l]{%
        \begin{tabular}{c|ccc} 
        \toprule[1.2pt]
        \textbf{$L$} & \textbf{mAP$_{50}$} & \textbf{mAP$_{75}$} & \textbf{mAP} \\
        \midrule
        0  & 71.5 & 58.7 & 52.6 \\
        2  & 72.3 & 59.4 & 53.4 \\
        4  & 72.8 & 59.9 & 54.0 \\
        \rowcolor{LightPurple}
        8  & \best{73.1} & \best{60.2} & \best{54.3} \\
        16 & 72.9 & 60.0 & 54.1 \\
        \bottomrule[1.2pt]
        \end{tabular}%
        }
        \label{tab:fourier_freq}
    \end{minipage}

\end{figure}

\nbf{Rotation Group Order}
We compare $C_2$, $C_4$, and $C_8$ groups in Tab.~\ref{tab:group_order}.
Higher group order improves accuracy at the cost of increased parameters and FLOPs.
$C_4$ achieves the best balance between accuracy and efficiency: $C_8$ incurs more than $2\times$ FLOPs for only marginal improvement, suggesting that $90^\circ$ discretization already captures the dominant orientation variations.

\nbf{Fourier Frequency in SpeIW}
We ablate the Fourier positional encoding frequency $L$ in SpeIW (Tab.~\ref{tab:fourier_freq}).
Higher $L$ enables better high-frequency modeling along the spectral dimension, but excessive frequencies may lead to overfitting.
$L\!=\!8$ achieves the best performance.

\nbf{Spectral Sensitivity}
Tab.~\ref{tab:spectral_sensitivity} compares RGB/MSI inputs across MODA detectors.
MSI improves every method, with FressDet showing the largest gain (+2.7 mAP$_{50}$ and +3.7 mAP); this supports continuous spectral modeling as a better way to exploit the available bands more fully.

\nbf{Spectral-Spatial Routing}
Tab.~\ref{tab:routing_mechanism} shows that routing should match the branch property.
Soft spectral routing preserves diversity, while hard spatial routing keeps sharper boundaries; the mismatched alternatives lose one of these cues.

\nbf{Prototype Number in ReCoW}
We ablate the spatial grid size used to generate prototypes in ReCoW (Tab.~\ref{tab:recog_prototypes}), which shows that smaller grids lack sufficient spatial resolution to capture diverse local patterns, while larger grids introduce redundancy and overfitting, and the $4\!\times\!4$ grid ($K\!=\!16$) achieves the best balance.
\begin{table}[!t]
    \centering
    \begin{minipage}[t]{0.71\textwidth}
    \centering \scriptsize
        \captionsetup{justification=raggedright,singlelinecheck=false}
        \tabcaption{Spectral sensitivity on MODA. Entries report RGB/MSI results with gains in parentheses.}
        \setlength{\tabcolsep}{0.2pt}{}
        \begin{tabular}{l|cccc}
        \toprule[1.2pt]
        \textbf{Method} & \textbf{mAP$_{50}$} & \textbf{mAP} & \textbf{\#F}(G) & \textbf{\#P}(M) \\
        \midrule
        S-RCNN      & 65.2/66.1\gain{(+0.9)} & 39.6/41.0\gain{(+1.4)} & 227.4/231.8 & 45.13/45.15 \\
        O-RepPoints & 66.2/66.5\gain{(+0.3)} & 39.4/40.9\gain{(+1.5)} & 209.2/213.6 & 36.83/36.85 \\
        OSSDet      & 67.1/69.0\gain{(+1.9)} & 39.7/42.7\gain{(+3.0)} & 258.7/263.1 & 36.46/36.48 \\
        \rowcolor{LightPurple}
        FressDet & \best{70.4/73.1}\gain{(+2.7)} & \best{50.6/54.3}\gain{(+3.7)} & \best{31.0/31.7} & \best{2.29/2.30} \\
        \bottomrule[1.2pt]
        \end{tabular}
        \label{tab:spectral_sensitivity}
    \end{minipage}
    \hfill
    \begin{minipage}[t]{0.28\textwidth}
    \centering \scriptsize
        \captionsetup{justification=raggedright,singlelinecheck=false}
        \tabcaption{Routing ablation in ReCoW.}
        \setlength{\tabcolsep}{0.8pt}{}
        \begin{tabular}{cc|cc}
        \toprule[1.2pt]
        \textbf{Spe.} & \textbf{Spa.} & \textbf{mAP$_{50}$} & \textbf{mAP} \\
        \midrule
        Hard & Hard & 70.4 & 51.2 \\
        Soft & Soft & 71.6 & 52.4 \\
        Hard & Soft & 69.8 & 50.6 \\
        \rowcolor{LightPurple}
        Soft & Hard & \best{73.1} & \best{54.3} \\
        \bottomrule[1.2pt]
        \end{tabular}
        \label{tab:routing_mechanism}
    \end{minipage}

    \begin{minipage}[t]{0.318\textwidth}
    \centering \scriptsize
        \captionsetup{justification=raggedright,singlelinecheck=false}
        \tabcaption{Prototype grid ablation in ReCoW.}
        \setlength{\tabcolsep}{0.2pt}{}
        \begin{tabular}{cc|ccc}
        \toprule[1.2pt]
        \textbf{$K$} & \textbf{Grid} & \textbf{mAP$_{50}$} & \textbf{mAP$_{75}$} & \textbf{mAP} \\
        \midrule
        4  & $2\!\times\!2$ & 71.8 & 59.0 & 52.9 \\
        9  & $3\!\times\!3$ & 72.5 & 59.6 & 53.6 \\
        \rowcolor{LightPurple}
        16 & $4\!\times\!4$ & \best{73.1} & \best{60.2} & \best{54.3} \\
        25 & $5\!\times\!5$ & 72.8 & 59.9 & 54.0 \\
        \bottomrule[1.2pt]
        \end{tabular}
        \label{tab:recog_prototypes}
    \end{minipage}
    \hfill
    \begin{minipage}[t]{0.338\textwidth}
    \renewcommand\arraystretch{1.0}
    \centering \scriptsize
        \captionsetup{justification=raggedright,singlelinecheck=false}
        \tabcaption{Fusion strategy ablation in ReCoW.}
        \setlength{\tabcolsep}{0.2pt}{}
        \begin{tabular}{l|ccc}
        \toprule[1.2pt]
        \textbf{Fusion} & \textbf{mAP$_{50}$} & \textbf{mAP$_{75}$} & \textbf{mAP} \\
        \midrule
        Addition        & 71.9 & 59.0 & 52.8 \\
        Concat.   & 72.4 & 59.5 & 53.4 \\
        Softmax    & 72.7 & 59.8 & 53.8 \\
        \rowcolor{LightPurple}
        Agree. gate & \best{73.1} & \best{60.2} & \best{54.3} \\
        \bottomrule[1.2pt]
        \end{tabular}
        \label{tab:recog_fusion}
    \end{minipage}
    \hfill
    \begin{minipage}[t]{0.324\textwidth}
    \renewcommand\arraystretch{1.25}
    \centering \scriptsize
        \captionsetup{justification=raggedright,singlelinecheck=false}
        \tabcaption{Dual-branch ablation in ReCoW.}
        \setlength{\tabcolsep}{0pt}{}
        \begin{tabular}{cc|ccc}
        \toprule[1.2pt]
        \textbf{Spe.} & \textbf{Spa.} & \textbf{mAP$_{50}$} & \textbf{mAP$_{75}$} & \textbf{mAP} \\
        \midrule
        \cmark & \xmark & 72.3 & 59.5 & 53.6 \\
        \xmark & \cmark & 71.8 & 59.1 & 53.2 \\
        \cmark & \cmark & \best{73.1} & \best{60.2} & \best{54.3} \\
        \bottomrule[1.2pt]
        \end{tabular}
        \label{tab:recog_branch}
    \end{minipage}
\end{table}

\vspace{-1.0\baselineskip}
\enlargethispage{2\baselineskip}
\nbf{Branch Fusion Strategy in ReCoW}
We ablate the fusion strategy for combining the spectral and spatial branches in ReCoW (Tab.~\ref{tab:recog_fusion}).
Our learned gate outperforms all alternatives, confirming the benefit of adaptive branch selection over fixed or implicit fusion.
Against addition, concatenation, and softmax fusion, the gate raises mAP by 1.5, 0.9, and 0.5 points, respectively.

\nbf{Dual-Branch Design in ReCoW}
We ablate each ReCoW branch (Tab.~\ref{tab:recog_branch}).
The spectral and spatial branches achieve 72.3\% and 71.8\% mAP$_{50}$, respectively.
Their combination reaches 73.1\%, confirming that the two branches capture complementary information and yield consistent improvements when combined.

\nbf{Rotation Equivariance Error}
To quantify how well rotation equivariance is preserved, we measure the equivariance error at each backbone stage, defined as the RMS difference between rotating the input before and the output after the network~\cite{wu2025measuring}: Fig.~\ref{fig:equiv_error} shows FressDet maintains low equivariance error across all stages, while OSSDet exhibits consistently higher errors.

\section{Conclusion}

We present FressDet, the first fully rotation-equivariant spectral--spatial learning framework for multispectral object detection.
FressDet addresses discrete spectral processing, spectral--spatial reliability shifts, and arbitrary orientations through three components: SpeIW for continuous, monotonic spectral resampling, ReCoW for agreement-weighted residual refinement across pyramid levels, and an oriented-aware head for stable orientation prediction without parameter replication.
Extensive experiments on five public benchmarks demonstrate that FressDet achieves state-of-the-art accuracy with the fewest parameters among compared multispectral detectors and exhibits the best robustness even under input rotational perturbations.

\section*{Acknowledgements}
This work was financially supported by the National Natural Science Foundation of China (No. 62101032), the Chongqing Excellent Young Scientists Fund (No. CSTB2025NSCQ-JQX0017), and the High-Quality Development Special Project of the Ministry of Industry and Information Technology (TC240HAJ9-35).

\clearpage

\putbib[main]
\end{bibunit}

\clearpage
\hypersetup{pageanchor=false}
\begin{NoHyper}
\setcounter{page}{1}
\setcounter{section}{0}
\setcounter{subsection}{0}
\setcounter{figure}{0}
\setcounter{table}{0}
\setcounter{equation}{0}
\setcounter{footnote}{0}
\begin{bibunit}[splncs04]
\nocite{cascadercnn,cssa,chen2024weakly,cohen2016group,dai2017deformable,fu2024lraf,fu2024rotation,ImageBind,han2025moda,calnet,hod3k,jia2021llvip,yolov8,fred,lin2014coco,nie2024crossyolo,razakarivony2016vedai,shen2024icafusion,dronevehicle,tian2025dcccnet,wang2023internimage,e2cnn,DCMNet,xiong2024efficient,xu2024ssefft,vfnet,zhang2024dsmavd,zhang2024esmyolo,RSDet,zhu2025wavemamba,zhu2025mda,Zhu_2019_CVPR}

\title{Supplementary Material:\\Fully Rotation-Equivariant Spectral-Spatial Learning for Multispectral Object Detection}
\titlerunning{Supplementary Material}
\author{Peng Zhang\orcidlink{0009-0008-9585-9622} \and Tingfa Xu$^\dagger$\orcidlink{0000-0001-5452-2662} \and Shuaihao Han\orcidlink{0009-0008-2594-1516} \and Jianan Li$^\dagger$\orcidlink{0000-0002-6936-9485}}
\authorrunning{P. Zhang et al.}
\institute{Beijing Institute of Technology, Beijing, China}
\maketitle
\begingroup
\renewcommand{\thefootnote}{\dag}
\footnotetext{Correspondence to: Tingfa Xu and Jianan Li.}
\endgroup

\section{Equivariance Preliminaries}
\subsection{Algebraic and Numerical Equivariance}
Let $\Phi:V_{\mathrm{in}}\to V_{\mathrm{out}}$ map between feature spaces carrying representations $\rho^{\mathrm{in}}$ and $\rho^{\mathrm{out}}$ of a group $G$. Algebraic $G$-equivariance means
\begin{equation}
  \Phi\!\left(\rho^{\mathrm{in}}_g f\right)
  =\rho^{\mathrm{out}}_g\Phi(f),
  \qquad g\in G.
  \label{eq:supp_equiv}
\end{equation}
For a numerical implementation, we report the normalized residual
\begin{equation}
  \epsilon_{\mathrm{rel}}(\Phi,g,f)=
  \frac{\left\|\Phi(\rho^{\mathrm{in}}_g f)-\rho^{\mathrm{out}}_g\Phi(f)\right\|}
  {\max\!\left(\|\Phi(\rho^{\mathrm{in}}_g f)\|,\|\rho^{\mathrm{out}}_g\Phi(f)\|,\eta\right)},
  \label{eq:supp_relerr}
\end{equation}
where $\eta>0$ prevents division by zero. Unlike an unrestricted supremum over all $f$, Eq.~\eqref{eq:supp_relerr} is well defined on the evaluated finite feature set. In this paper, ``fully rotation-equivariant'' means that the learned dense predictor contains no deliberately symmetry-breaking stage under the modeled $C_4$ action. It does not mean bitwise equality for every floating-point input. Kernel rotation, interpolation, reductions, and finite precision can leave round-off-level residuals even when the corresponding real-arithmetic operator is equivariant.

For a composition $\Phi_L\circ\cdots\circ\Phi_1$ on a bounded feature domain, numerical deviations admit the usual Lipschitz propagation bound. The useful conclusion is therefore conditional: algebraically equivariant components compose equivariantly, while measured floating-point residuals should be reported rather than asserted to be identically zero.

\subsection{Lattice and Boundary Conditions}
The spatial action must map the sampled output lattice to itself. For the cardinal rotations in $C_4$, this is satisfied by the square feature grids used by FressDet and by stride/kernel choices whose sampling origins remain aligned after rotation. Local ReCoW windows additionally require the padded grid and its window partition to be permuted by the same $C_4$ action.

Constant zero extension does not inherently break $C_4$ equivariance. Let $P_0$ symmetrically extend a square feature map by the same constant on every side. Because a cardinal rotation maps the padded square to itself and leaves the constant exterior unchanged,
\begin{equation}
  P_0(\rho_g f)=\rho_g P_0(f), \qquad g\in C_4.
  \label{eq:supp_padding}
\end{equation}
FressDet uses constant zero padding in its convolutions and, when needed, in ReCoW window padding. Equation~\eqref{eq:supp_padding}, rather than a reflect-padding assumption, is the relevant boundary condition. Asymmetric padding or an incompatible rectangular lattice would require a separate analysis.

\subsection{Pointwise, Normalization, and Pooling Operations}
Pointwise nonlinearities commute with the regular action because that action only reindexes spatial and group coordinates. For a group feature $F(x,h)$,
\begin{equation}
  \sigma(\rho_g F)(x,h)
  =\sigma(F(g^{-1}x,g^{-1}h))
  =(\rho_g\sigma(F))(x,h).
\end{equation}
The same observation applies to componentwise clipping. Group-shared normalization preserves the action when its affine parameters and reduction axes are shared consistently across group slots.

Group averaging removes the orientation index while retaining the transformed spatial coordinate:
\begin{equation}
  \overline{F}(x)=\frac{1}{|G|}\sum_{h\in G}F(x,h),
  \qquad
  \overline{\rho_gF}(x)=\overline{F}(g^{-1}x).
  \label{eq:supp_groupmean}
\end{equation}
Average or max pooling applied independently to each group slot is equivariant whenever its sampled lattice satisfies the alignment condition above. These facts support PatchMerging, ReSPPF, and the invariant descriptors used by ReCoW.

\section{Model Architecture}
FressDet follows a four-stage backbone, a bidirectional feature pyramid, and an oriented-aware head. Every learned feature map before the final readout is represented as $(B,C,|G|,H,W)$ with $G=C_4$.

\paragraph{Backbone.}
The input is lifted to a regular group feature by a rotation-tied lifting convolution. Each stage contains group layer normalization, Spectral Implicit Warp (SpeIW), a residual connection, and an equivariant MLP. SpeIW forms a spatial latent code with an equivariant pointwise projection. Its coordinate encoder contains the spectral coordinate together with sinusoidal and cosine features, and a shared MLP decodes coordinate-dependent bases. Positive spectral steps are cumulatively integrated to construct an ordered sampling map. PatchMerging performs spatial average-pooling downsampling in two equivariant pointwise branches, and ReSPPF aggregates context with per-group max pooling followed by an equivariant pointwise projection.

\paragraph{Neck.}
The neck uses equivariant transposed convolution for upsampling, equivariant depthwise-pointwise convolution for downsampling, channel concatenation, and ReCoW refinement. ReCoW contains two complementary local routing branches. The spectral branch computes group-averaged vector descriptors, cosine soft assignments to pooled prototypes, and confidence-weighted prototype reconstructions. The spatial branch routes a group-averaged scalar score by Gaussian affinity, forms per-group cluster means, and weights them by the winner margin. Equivariant post-projections transform both routed features. Their agreement controls a residual modulation of the original group feature; the released operator is not a scalar convex combination of the two routed branches.

\paragraph{Oriented-aware head.}
Each prediction scale uses two equivariant convolution blocks per branch. Classification averages over group slots before a shared $1\times1$ readout. Box regression keeps the group axis and uses a cyclic-tied DFL projection with a side-shared bias. The angle branch predicts one residual candidate per group slot and applies a soft circular decoder.

\section{Implementation-Aligned Equivariance Analysis}
We now analyze the operators implemented in the public release. The statements below concern the dense learned predictor and the $C_4$ action, under compatible square lattices and finite arithmetic inputs. Numerical equality is understood up to floating-point residuals.

\subsection{Spectral Implicit Warp}
Let $F\in\R^{C\times|G|\times H\times W}$. An equivariant projection and a shared linear map produce $a(x,g)\in\R^R$. For spectral coordinate $\lambda_c$, the shared coordinate network produces $b_c\in\R^R$, and
\begin{equation}
  r_c(x,g)=\langle a(x,g),b_c\rangle,
  \qquad
  d_c(x,g)=\operatorname{softplus}(r_c(x,g))+\varepsilon_s.
  \label{eq:supp_speiw_raw}
\end{equation}
The released implementation computes
\begin{align}
  u_c &= \sum_{j=1}^{c}d_j-d_1, \\
  \phi_c &= (C-1)\frac{u_c}{u_C+\varepsilon_d},
  \label{eq:supp_speiw_index}
\end{align}
then clamps $\phi_c$ to $[0,C-1]$. For finite real-arithmetic values, $d_c>0$, $u_1=0$, and $u_{c+1}>u_c$, so the sampling locations are ordered. Because the denominator contains $\varepsilon_d>0$, the implemented endpoint is slightly below $C-1$ rather than exactly equal to it. In floating point, rounding may collapse extremely small adjacent increments; the correct implementation-level claim is therefore ordered and numerically non-decreasing, not an unconditional bitwise strict-order theorem.

More explicitly, for $c'>c$ the index difference has the sign of a sum of positive steps:
\begin{equation}
  \phi_{c'}-\phi_c
  =(C-1)\frac{\sum_{j=c+1}^{c'}d_j}{u_C+\varepsilon_d}>0
  \label{eq:supp_speiw_order}
\end{equation}
in exact arithmetic. Thus the denominator offset changes the endpoint but not the channel order. Clipping only enforces the legal interpolation interval.

All operations in Eqs.~\eqref{eq:supp_speiw_raw}--\eqref{eq:supp_speiw_index} are shared across group slots and act independently at each $(x,g)$. If $\tilde{x}=h^{-1}x$ and $\tilde{g}=h^{-1}g$, then
\begin{equation}
  \phi_c(\rho_hF)(x,g)=\phi_c(F)(\tilde{x},\tilde{g}).
\end{equation}
Linear interpolation gathers only along the spectral channel axis, so it commutes with the spatial-group reindexing. Clipping also commutes with permutations. The final equivariant pointwise projection therefore yields
\begin{equation}
  \operatorname{SpeIW}(\rho_hF)=\rho_h\operatorname{SpeIW}(F)
\end{equation}
in the algebraic model, with round-off-level residuals in floating-point execution.

To see the interpolation step directly, let $\mathcal{I}$ denote the two-neighbor gather and linear blend. With $\tilde{x}=h^{-1}x$ and $\tilde{g}=h^{-1}g$,
\begin{align}
 \mathcal{I}\!\left[(\rho_hF)(x,g,\cdot);\phi_c(\rho_hF)(x,g)\right]
 &=\mathcal{I}\!\left[F(\tilde{x},\tilde{g},\cdot);\phi_c(F)(\tilde{x},\tilde{g})\right].
\end{align}
The sampling indices and both interpolation weights are consequently reindexed together.

\subsection{Rotation-Equivariant Consistency Weighting}
ReCoW operates on square local windows. Let $m(x)=|G|^{-1}\sum_gF(x,g)$ denote the group-averaged vector descriptor and let $\{c_k\}_{k=1}^{K}$ be its adaptively pooled prototypes within a window.

\paragraph{Spectral routing.}
For each location, the spectral branch computes
\begin{align}
  p_k(x)&=\operatorname{softmax}_{k}\!\left(\tau\cos(c_k,m(x))\right),\\
  r_{\mathrm{spec}}(x)&=\sum_k p_k(x)c_k.
\end{align}
The reconstruction is multiplied by two confidence factors: the probability-weighted prototype similarity and the cosine agreement between $r_{\mathrm{spec}}(x)$ and $m(x)$, each mapped to $[0,1]$. It is then broadcast over the group axis and passed through an equivariant pointwise post-projection, producing $S(x,g)$.
Writing these factors explicitly,
\begin{align}
 c_{\mathrm{assign}}(x)&=\frac{1+\sum_kp_k(x)\cos(c_k,m(x))}{2},\\
 c_{\mathrm{agree}}(x)&=\frac{1+\cos(r_{\mathrm{spec}}(x),m(x))}{2}.
 \label{eq:supp_spec_conf}
\end{align}
The routed vector before post-projection is $c_{\mathrm{assign}}c_{\mathrm{agree}}r_{\mathrm{spec}}$, shared over the orientation slots exactly as in the release.

\paragraph{Spatial routing.}
The spatial score is the mean over channels and group slots. Pooled scalar centers are compared with normalized pixel scores by a Gaussian affinity. The winning center is selected by \texttt{argmax}; original group features assigned to that center are averaged separately for every group slot. The selected cluster feature is weighted by a sigmoid of the top-one/top-two affinity margin and passed through an equivariant depthwise post-projection, producing $T(x,g)$.
For affinity values $a_k(x)$, let $a_{(1)}(x)\geq a_{(2)}(x)$ be the two largest values. The released confidence is $c_{\mathrm{spat}}(x)=\sigma(\kappa[a_{(1)}(x)-a_{(2)}(x)])$. If $\Omega_k$ is the set assigned to center $k$, the routed group feature is
\begin{equation}
 r_{\mathrm{spat}}(x,g)=c_{\mathrm{spat}}(x)
 \frac{\sum_{x'\in\Omega_{k^*(x)}}F(x',g)}{\max(1,|\Omega_{k^*(x)}|)}.
 \label{eq:supp_spat_route}
\end{equation}

Under the $C_4$ action, group averaging gives spatially transformed invariant descriptors. On a compatible square window lattice, rotation permutes windows, within-window sites, and prototype locations. Soft assignments, pooled reconstructions, confidence factors, and per-group cluster means therefore transform consistently. For hard routing, this conclusion assumes a unique winning prototype; an exact affinity tie is resolved by implementation order and is not covered by a strict algebraic claim.

\paragraph{Agreement-weighted residual modulation.}
The public ReCoW forward function forms
\begin{align}
  A(x)&=\frac{1+\cos(\overline{S}(x),\overline{T}(x))}{2},\\
  M(x,g)&=A(x)\,\sigma(S(x,g)+T(x,g)),\\
  F'(x,g)&=F(x,g)+F(x,g)\odot M(x,g),
  \label{eq:supp_recow}
\end{align}
where bars denote group averages and $\odot$ is componentwise multiplication. This is an agreement-weighted residual gate on the input feature. It is not the convex interpolation $\beta S+(1-\beta)T$ stated in the earlier supplement.

Because $S$ and $T$ carry the regular representation, $S+T$ and its pointwise sigmoid do as well. The scalar field $A$ is group invariant in the orientation index and spatially equivariant. Hence $M$ carries the regular representation, and sums and componentwise products of tensors carrying the same permutation action preserve that action. Under the lattice and unique-routing conditions,
\begin{equation}
  \operatorname{ReCoW}(\rho_hF)=\rho_h\operatorname{ReCoW}(F)
\end{equation}
in exact arithmetic. Floating-point reductions and the generated rotated kernels introduce small numerical residuals.

Indeed, $A(\rho_hF)(x)=A(F)(h^{-1}x)$ and $M(\rho_hF)(x,g)=M(F)(h^{-1}x,h^{-1}g)$. Substitution into Eq.~\eqref{eq:supp_recow} reindexes both multiplicands identically. This is why the implemented tensor gate is equivariant even though it is not invariant over the group axis.

\subsection{Oriented-Aware Head}
Let $F_n$ be a neck feature. The classification stack is equivariant and its group mean readout satisfies
\begin{equation}
  f_{\mathrm{cls}}(\rho_hF_n)(x)=f_{\mathrm{cls}}(F_n)(h^{-1}x).
\end{equation}
The box stack retains group slots. Its cyclic-tied weights depend only on the relative offset between a box side and an input group slot, while one bias vector is shared across all four sides. A cyclic shift of group slots therefore induces the corresponding cyclic shift of side logits. The side-shared bias is essential: four independent side biases would remain fixed while the sides permute, violating the same law after training.

For the angle branch, let $q_g(x)$ be the group-indexed logit, $\omega=2\pi/|G|$, $\delta_g=(\sigma(q_g)-\tfrac12)\omega$, and $w_g=\operatorname{softmax}_g(q_g)$. The decoder forms
\begin{equation}
  z(x)=\sum_{g\in G}w_g(x)\exp\!\left(i[\delta_g(x)+\omega\operatorname{idx}(g)]\right),
  \qquad
  \hat\theta(x)=\operatorname{atan2}(\Im z(x),\Re z(x)).
  \label{eq:supp_angle}
\end{equation}
If $z(x)\neq0$, cyclic reindexing rotates $z$ by the group angle, and the wrapped OBB angle obeys the required $\pi$-periodic shift law. If $z(x)=0$, the circular mean is mathematically undefined; at a near-zero resultant, floating-point tie breaking can dominate. The angle proposition therefore requires a nonzero resultant rather than quantifying over every possible input.
For a group element $h$, the cyclic shift gives
\begin{equation}
  z_{\rho_hF}(x)=\exp\!\left(i\omega\operatorname{idx}(h)\right)z_F(h^{-1}x).
  \label{eq:supp_angle_shift}
\end{equation}
When the right-hand side is nonzero, taking its argument yields the desired additive angle shift before $\pi$-periodic wrapping.

\subsection{End-to-End Statement}
Consider the dense predictor from lifting through the raw class, box-distribution, and angle outputs. In exact arithmetic it is $C_4$-equivariant under the following explicit conditions: (i) the spatial lattices and ReCoW window partitions are compatible with the cardinal rotations; (ii) routed tensors are finite; (iii) hard spatial routing has a unique winner; and (iv) the angle resultant in Eq.~\eqref{eq:supp_angle} is nonzero. The proof follows by composition of the group convolutions, SpeIW, aligned pooling/downsampling, ReCoW, and the three structured readouts analyzed above.

The implementation is numerically equivariant rather than bitwise exact because rotated filters are generated with floating-point sampling and because reductions and interpolation accumulate rounding error. Box decoding and non-maximum suppression operate in image coordinates and are outside this dense-predictor statement.

\paragraph{Rotation robustness evaluation.}
For Fig.~1(b) of the main paper, each MODA test image is rotated from $0^\circ$ to $359^\circ$ in $1^\circ$ increments using bilinear interpolation about the image center. The canvas is expanded to retain the rotated content and the exterior is filled with zeros. Ground-truth centers and angles are transformed consistently, while detection settings remain fixed. This experiment evaluates empirical robustness to arbitrary angles; it is distinct from the algebraic $C_4$ statement above.

\section{Datasets}
\subsection{Primary Benchmarks}
\paragraph{MODA.}
MODA is a large-scale multispectral aerial object detection dataset collected across 50 urban areas under varying illumination. It contains 14,041 images at $1200\times900$, eight spectral bands, and 330,191 oriented-box annotations over eight categories. We use the standard 9,156/4,885 train/test split.

\paragraph{HOD3K.}
HOD3K contains 3,242 natural-scene multispectral images at $512\times256$, 16 spectral bands, and 15,149 annotations over people, bike, and car. We follow the standard split of 2,308 training, 219 validation, and 715 test images.

\paragraph{DroneVehicle.}
It contains 28,439 RGB-infrared image pairs. Its 953,087 oriented annotations cover car, bus, truck, freight-car, and van. We use 17,990 pairs for training, 1,469 for validation, and 8,980 for testing.

\subsection{Extended Benchmarks for Generalization}
\paragraph{LLVIP.}
LLVIP is a well-aligned RGB-thermal pedestrian dataset captured at night from a surveillance perspective. It contains 15,488 pairs at $1280\times1024$; we use the official 12,025/3,463 train/test split.

\paragraph{VEDAI.}
VEDAI provides 1,210 aerial RGB-infrared pairs with oriented annotations over nine vehicle categories and is dominated by small objects. Following common practice in the absence of an official partition, we use a 1,089/121 train/test split.

\section{Implementation}
\subsection{Evaluation Metrics}
We use the COCO protocol and report $\mathrm{mAP}_{50}$, $\mathrm{mAP}_{75}$, and mAP averaged over IoU thresholds from 0.50 to 0.95. Rotated-rectangle IoU is used for MODA, DroneVehicle, and VEDAI; axis-aligned IoU is used for HOD3K and LLVIP. FLOPs and parameter counts include the group dimension.

\subsection{Training Settings}
Experiments use the Ultralytics YOLO training pipeline. Detailed settings and extended benchmark results begin on the following page.

\clearpage
\setcounter{page}{11}
\renewcommand{\thesection}{S\arabic{section}}
\renewcommand{\thesubsection}{S\arabic{section}.\arabic{subsection}}
\renewcommand{\thetable}{\arabic{table}}
\renewcommand{\thefigure}{\arabic{figure}}
\renewcommand{\theequation}{\arabic{equation}}
\label{sec:impl_hyper}
\begingroup
\makeatletter
\def\@currentlabel{S1}
\label{sec:prelim}
\makeatother
\endgroup
\setcounter{table}{0}
\setcounter{figure}{0}
\setcounter{equation}{26}
\setcounter{proposition}{5}

\begin{table}[!t]
\centering
\scriptsize
\setlength{\tabcolsep}{25pt}
\caption{Common training hyper-parameters of FressDet.}
\label{tab:hyper_common}
\begin{tabular}{l|c}
\toprule
\textbf{Hyper-parameter} & \textbf{Value} \\
\midrule
\midrule
Optimizer & AdamW \\
Initial learning rate & 0.01 \\
Final learning rate & 0.0001 \\
Learning rate schedule & linear decay \\
Momentum & 0.937 \\
Weight decay & 0.0005 \\
Warm-up epochs & 3 \\
Warm-up momentum & 0.8 \\
Warm-up bias learning rate & 0.1 \\
Box loss gain & 7.5 \\
Class loss gain & 0.5 \\
DFL loss gain & 1.5 \\
\bottomrule
\end{tabular}
\end{table}

\subsection{Dataset-specific Settings}
\label{sec:impl_dataset}

For multi-modal datasets (DroneVehicle, LLVIP, VEDAI), thermal images capture temperature information in the infrared spectrum and are inherently robust to illumination variations; they are thus often combined with RGB images to improve detection under poor lighting.
Existing RGBT methods predominantly exploit complementary spatial information via two-stream networks that separately extract and fuse features from RGB and thermal modalities.
In contrast, our approach integrates spectral-spatial information within a single-stream framework by concatenating RGB and thermal images along the channel dimension to form a four-channel spectral input.
Although multi-modal inputs are sparsely sampled, SpeIW learns a continuous feature representation along the spectral dimension, enabling unified spectral-spatial reasoning rather than treating each modality as an independent stream.
Unlike prior methods that rely on predefined rotation augmentations to achieve orientation robustness, FressDet leverages built-in rotation equivariance and requires no such augmentation during training, simplifying the pipeline while ensuring fair comparison across all datasets.

\nbf{(i) MODA}
Training was performed with a batch size of 8 for 20 epochs at an input size of $1216\!\times\!928$ with 8 spectral bands, following~\cite{han2025moda}.
For the comparison in Tab.~1 of the main paper, all methods are evaluated on the official MODA split under the same input resolution and the same COCO-style rotated-box evaluation protocol.
We directly use the provided 8-band MODA images as input, without additional handcrafted band preprocessing.

\nbf{(ii) HOD3K}
We used a batch size of 8 for 14 epochs at an input size of $1024\!\times\!512$ with 16 spectral bands, following~\cite{han2025moda}.

\nbf{(iii) DroneVehicle}
RGB and infrared pairs at $640\!\times\!640$ with batch size 16, following~\cite{zhu2025wavemamba} for data settings; we train for 200 epochs.
Following common practice, we use the infrared annotations as training labels due to the more comprehensive target annotations available in the IR modality~\cite{chen2024weakly}.

\nbf{(iv) LLVIP}
Visible and infrared pairs at $1280\!\times\!1024$ with batch size 4, with other settings following~\cite{han2025moda}.

\nbf{(v) VEDAI}
Infrared and visible pairs at $640\!\times\!640$ with batch size 16, following~\cite{zhu2025wavemamba} for data settings; we train for 200 epochs.

\section{Additional Experiments}
\label{sec:exp}

\subsection{Extended Benchmarks}
\label{sec:exp_extended}

\nbf{Discussion on MODA}
The gains on MODA are category-dependent rather than uniform across all classes.
FressDet achieves substantial improvements on bike (+16.7\%) and pedestrian (+26.8\%), while trailing OSSDet on truck and tricycle.
Despite this variation, FressDet achieves the highest overall mAP (54.3 vs.\ 42.7) with only 6.3\% of OSSDet's parameters, demonstrating a favorable accuracy--efficiency trade-off.
On other benchmarks, the absolute mAP gains are smaller as baseline methods already achieve relatively high accuracy.
In such regimes, FressDet's advantage is comparable or better accuracy with substantially fewer parameters and improved rotational robustness (Fig.~1(b) of the main paper).

\nbf{Results on LLVIP}
Tab.~\ref{tab:llvip} shows that FressDet achieves the best overall performance on LLVIP, reaching 65.1 mAP.
Compared with the previous best method, OSSDet, FressDet improves mAP by 1.4 points while reducing the model size from 36.6M to 2.3M parameters and lowering the computational cost from 259.3G to 34.6G FLOPs.
All LLVIP images are captured at night, making this dataset suitable for evaluating model performance under low-illuminance conditions where spatial information from visible images is severely limited.
Despite such challenging scenarios, FressDet effectively leverages the learned features to detect pedestrians that are difficult to distinguish in visible images alone, demonstrating the effectiveness and robustness of continuous spectral--spatial representation under poor illumination.
Fig.~\ref{fig:vis_llvip} shows qualitative detection results, where FressDet produces predictions that closely match the ground truth.

\begin{figure}[!t]
\centering
\includegraphics[width=\linewidth]{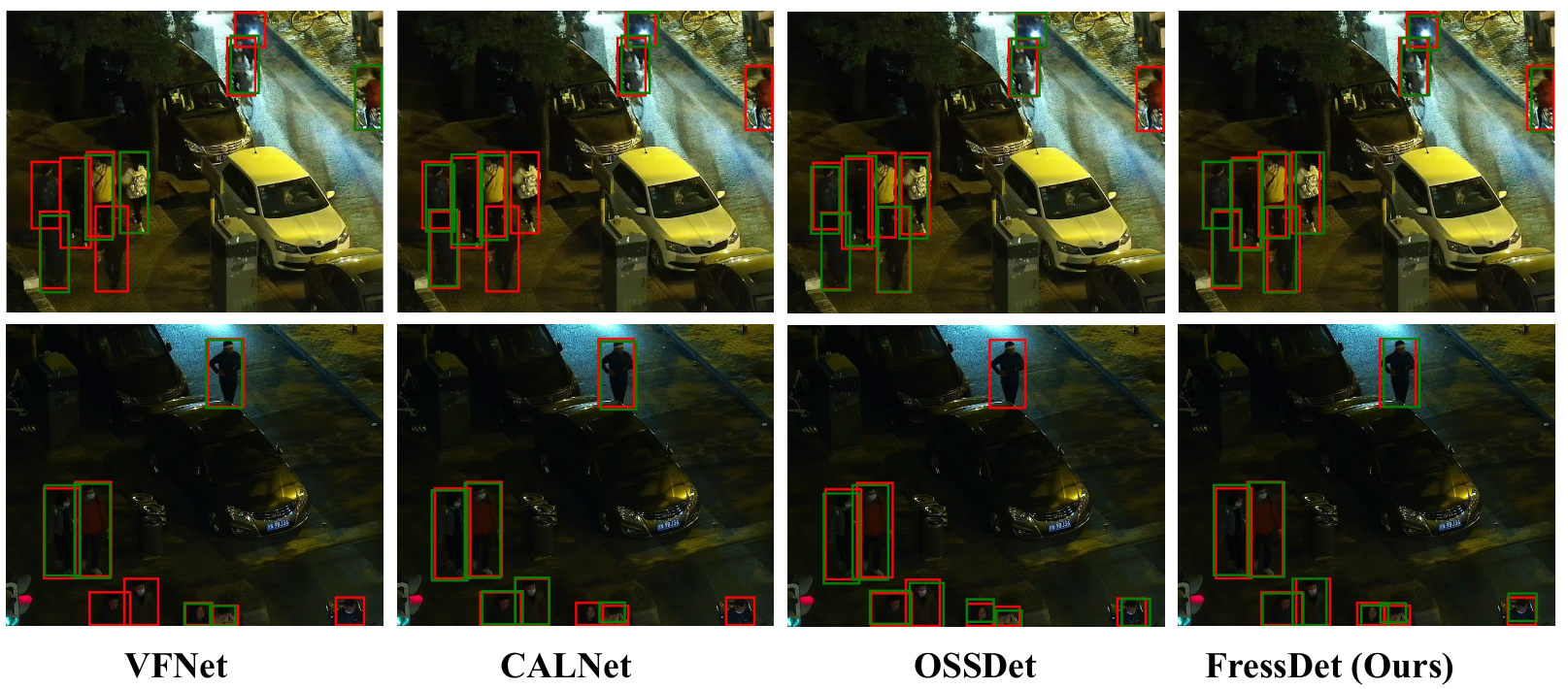}
\caption{Detection results on LLVIP. Red boxes indicate the ground truth, while green boxes denote the predictions.}
\label{fig:vis_llvip}
\end{figure}

\begin{table}[!t]
\centering
\scriptsize
\setlength{\tabcolsep}{3.5pt}
\caption{Performance comparisons on LLVIP. ``S'' denotes a single-stream framework, and ``T'' denotes a two-stream framework.}
\label{tab:llvip}
\begin{tabular}{l|ccccc}
\toprule
\textbf{Method} & \textbf{Modality} & \textbf{Type} & \textbf{mAP} & \textbf{FLOPs} & \textbf{Params} \\
\midrule
C-RCNN~\cite{cascadercnn}   & RGB    & S & 47.0 & 294.9G & 69.2M \\
VFNet~\cite{vfnet}          & RGB    & S & 47.3 & 247.8G & 32.7M \\
C-RCNN~\cite{cascadercnn}   & IR     & S & 56.8 & 294.9G & 69.2M \\
VFNet~\cite{vfnet}          & IR     & S & 61.5 & 247.8G & 32.7M \\
CSSA~\cite{cssa}            & RGB+IR & T & 59.2 & 357.1G & 66.9M \\
C-RCNN~\cite{cascadercnn}   & RGB+IR & S & 59.6 & 295.9G & 69.2M \\
RSDet~\cite{RSDet}          & RGB+IR & T & 61.3 & --     & --    \\
DCMNet~\cite{DCMNet}        & RGB+IR & T & 61.5 & 385.1G & 94.7M \\
VFNet~\cite{vfnet}          & RGB+IR & S & 61.8 & 248.8G & 32.7M \\
ImageBind~\cite{ImageBind}  & RGB+IR & T & 63.4 & --     & --    \\
CALNet~\cite{calnet}        & RGB+IR & T & 63.4 & 263.3G & 153.6M \\
OSSDet~\cite{han2025moda}   & RGB+IR & S & 63.7 & 259.3G & 36.6M \\
\midrule
\rowcolor{LightPurple}
FressDet (Ours) & RGB+IR & S & \textbf{65.1} & \textbf{34.6G} & \textbf{2.3M} \\
\bottomrule
\end{tabular}
\end{table}

\begin{table}[t]
\centering
\scriptsize
\setlength{\tabcolsep}{13.2pt}
\caption{Performance comparisons on VEDAI.}
\label{tab:vedai}
\begin{tabular}{l|ccc}
\toprule
\textbf{Method} & \textbf{mAP$_{50}$} & \textbf{mAP} & \textbf{Params} \\
\midrule
WaveMamba~\cite{zhu2025wavemamba} & 87.9 & 59.4 & 193.2M \\
LRAF-Net~\cite{fu2024lraf}        & 85.9 & 59.1 & --     \\
ICAFusion~\cite{shen2024icafusion}& 84.8 & 56.6 & 164.3M \\
SSE+FFT~\cite{xu2024ssefft}       & 86.5 & 56.9 & --     \\
ESM-YOLO~\cite{zhang2024esmyolo}  & 82.4 & --   & 80.2M  \\
DCCCNet~\cite{tian2025dcccnet}    & 82.0 & 49.9 & --     \\
CrossYOLO~\cite{nie2024crossyolo} & 79.8 & --   & --     \\
WaveMamba~\cite{zhu2025wavemamba} & 88.2 & 59.6 & 45.6M  \\
YOLOv8l-IR~\cite{yolov8}          & 73.6 & 52.2 & 43.7M  \\
YOLOv8l-RGB~\cite{yolov8}         & 62.9 & 46.5 & 43.7M  \\
MDA~\cite{zhu2025mda}             & 82.2 & --   & 72.4M  \\
DSM-AVD~\cite{zhang2024dsmavd}    & 79.3 & 50.3 & --     \\
WaveMamba~\cite{zhu2025wavemamba} & 88.4 & 59.8 & 69.1M  \\
\midrule
\rowcolor{LightPurple}
FressDet (Ours) & \textbf{89.7} & \textbf{60.2} & \textbf{2.3M} \\
\bottomrule
\end{tabular}
\end{table}

\nbf{Results on VEDAI}
Tab.~\ref{tab:vedai} summarizes the comparison results on VEDAI.
FressDet achieves the best performance with 89.7 mAP$_{50}$ and 60.2 mAP, while using only 2.3M parameters.
Compared with the strongest WaveMamba variant, FressDet improves mAP$_{50}$ by 1.3 points and mAP by 0.4 points, with substantially lower model complexity.

\subsection{Ablation Studies}
\label{sec:exp_ablation}

\nbf{Ablation on Routing Mechanism in ReCoW}
The design of ReCoW adopts soft assignment for the spectral branch and hard assignment for the spatial branch.
To validate this design choice, we conduct a systematic ablation by testing all four combinations of routing mechanisms.
Tab.~\ref{tab:routing_ablation} reports the results on MODA.

\begin{table}[t]
\centering
\scriptsize
\setlength{\tabcolsep}{2.5pt}
\caption{Ablation on routing mechanisms in ReCoW. ``Soft'' denotes cosine-based soft assignment (Eq.~10), and ``Hard'' denotes Gaussian-affinity hard assignment (Eq.~12).}
\label{tab:routing_ablation}
\begin{tabular}{cc|ccc|l}
\toprule
\textbf{Spectral} & \textbf{Spatial} & \textbf{mAP$_{50}$} & \textbf{mAP$_{75}$} & \textbf{mAP} & \textbf{Observation} \\
\midrule
Hard & Hard & 70.4 & 56.8 & 51.2 & spectral diversity lost \\
Soft & Soft & 71.6 & 57.9 & 52.4 & spatial boundaries blurred \\
Hard & Soft & 69.8 & 56.1 & 50.6 & worst of both \\
\rowcolor{LightPurple}
Soft & Hard & \textbf{73.1} & \textbf{60.2} & \textbf{54.3} & complementary strengths \\
\bottomrule
\end{tabular}
\end{table}

The soft-spectral + hard-spatial configuration achieves the best performance across all metrics.
The gap is most pronounced at mAP$_{75}$ (+2.3 over Soft-Soft), suggesting that the design primarily benefits localization precision.

Spectral signatures in multispectral imagery often exhibit continuous mixing due to sub-pixel blending.
Soft assignment allows each location to aggregate from multiple prototypes, which is consistent with this characteristic; switching to hard assignment drops mAP by 3.1 (Hard-Hard vs.\ Soft-Hard).
In contrast, spatial grouping benefits from crisp boundaries that respect object contours.
Hard assignment enforces discrete region membership, whereas soft assignment tends to blur adjacent regions---the degraded mAP$_{75}$ of Soft-Soft (57.9 vs.\ 60.2) is consistent with this interpretation.
The Hard-Soft variant performs worst (50.6 mAP), as it discards spectral diversity while blurring spatial boundaries.

\nbf{Ablation on Backbone Architecture}
To verify that the proposed SpeIW and ReCoW modules are not tied to a specific backbone, we replace SpeIWNet with ReResNet50, a rotation-equivariant ResNet variant based on e2cnn~\cite{e2cnn}.
Tab.~\ref{tab:backbone_ablation} reports the results across all five benchmarks.

\begin{table}[t]
\centering
\scriptsize
\setlength{\tabcolsep}{7.5pt}
\caption{Ablation on backbone architecture.}
\label{tab:backbone_ablation}
\begin{tabular}{l|l|ccc}
\toprule
\textbf{Dataset} & \textbf{Backbone} & \textbf{mAP$_{50}$} & \textbf{mAP} & \textbf{Params} \\
\midrule
\multirow{2}{*}{MODA~\cite{han2025moda}} & ReResNet50 & 70.8 & 52.7 & 26.1M \\
& \cellcolor{LightPurple}SpeIWNet & \cellcolor{LightPurple}\textbf{73.1} & \cellcolor{LightPurple}\textbf{54.3} & \cellcolor{LightPurple}\textbf{2.3M} \\
\midrule
\multirow{2}{*}{HOD3K~\cite{hod3k}} & ReResNet50 & 91.5 & 59.6 & 26.2M \\
& \cellcolor{LightPurple}SpeIWNet & \cellcolor{LightPurple}\textbf{93.8} & \cellcolor{LightPurple}\textbf{61.1} & \cellcolor{LightPurple}\textbf{2.3M} \\
\midrule
\multirow{2}{*}{DroneVehicle~\cite{dronevehicle}} & ReResNet50 & 78.1 & 60.2 & 26.1M \\
& \cellcolor{LightPurple}SpeIWNet & \cellcolor{LightPurple}\textbf{80.3} & \cellcolor{LightPurple}\textbf{61.6} & \cellcolor{LightPurple}\textbf{2.3M} \\
\midrule
\multirow{2}{*}{LLVIP~\cite{jia2021llvip}} & ReResNet50 & 95.4 & 63.8 & 26.1M \\
& \cellcolor{LightPurple}SpeIWNet & \cellcolor{LightPurple}\textbf{97.6} & \cellcolor{LightPurple}\textbf{65.1} & \cellcolor{LightPurple}\textbf{2.3M} \\
\midrule
\multirow{2}{*}{VEDAI~\cite{razakarivony2016vedai}} & ReResNet50 & 87.4 & 58.9 & 26.1M \\
& \cellcolor{LightPurple}SpeIWNet & \cellcolor{LightPurple}\textbf{89.7} & \cellcolor{LightPurple}\textbf{60.2} & \cellcolor{LightPurple}\textbf{2.3M} \\
\bottomrule
\end{tabular}
\end{table}

Across all datasets, SpeIWNet consistently outperforms ReResNet50 while using only 9\% of the parameters.
The performance gap is most pronounced on MODA (+2.3 mAP$_{50}$, +1.6 mAP).
ReResNet50, designed as a general-purpose equivariant backbone, lacks explicit spectral reasoning and treats each channel independently.
These results confirm that the lightweight SpeIWNet, with its spectral-aware design, is more effective for multispectral object detection than simply scaling up a generic equivariant backbone.

\section{Rotation-Equivariant DCN Variants}
\label{sec:dcn}

For the ablation in Tab.~6 of the main paper, we replace SpeIW with a family of spectral deformable operators adapted from DCN~\cite{dai2017deformable,Zhu_2019_CVPR,wang2023internimage,xiong2024efficient}.
The key adaptation is that the deformation is applied \emph{along the ordered spectral dimension}, rather than on the 2D spatial convolution grid.
This design yields a fair comparison to SpeIW: both operators resample spectral responses from rotation-equivariant group features, but the DCN-style baselines retain the \emph{discrete offset-and-interpolation} inductive bias of deformable convolution, whereas SpeIW models a continuous monotone warp over the entire spectral dimension.

\subsection{Spectral Deformable Resampling}
\label{sec:dcn_spectral}

Let $F\in\mathbb{R}^{C\times |G|\times H\times W}$ be the input group feature map, where $C$ is the number of spectral channels and $G=\mathrm{C}_{N}$ is the cyclic rotation group.
Standard DCN deforms a spatial sampling grid by learning offsets for each kernel location.
In our spectral adaptation, the regular grid is the ordered channel index set $\{0,1,\dots,C{-}1\}$.
For each spectral channel $c$, spatial location $x$, and group element $g$, we predict offsets from the input group feature using an equivariant predictor:
\begin{equation}
  \Delta_{c,k}(x,g)=\Gamma_{\Delta,k}(F)(x,g),
  \qquad k=1,\dots,K,
  \label{eq:spec_dcn_offset}
\end{equation}
where $K$ is the number of sampling points and $\Gamma_{\Delta,k}$ shares parameters across orientations.
Given a reference 1D stencil $\mathcal{R}=\{r_k\}_{k=1}^{K}$ on the spectral dimension, the deformed spectral positions are
\begin{equation}
  \phi_{c,k}(x,g)=c+r_k+\Delta_{c,k}(x,g).
  \label{eq:spec_dcn_loc}
\end{equation}
Because $\phi_{c,k}(x,g)$ is generally fractional, the sampled value is obtained by linear interpolation along the spectral dimension:
\begin{equation}
  \tilde{F}_{c,k}(x,g)
  = \mathcal{I}\!\bigl[F(x,g,\cdot);\phi_{c,k}(x,g)\bigr],
  \label{eq:spec_dcn_interp}
\end{equation}
where $\mathcal{I}[\,\cdot\,;t\,]$ denotes 1D linear interpolation at continuous spectral location $t$.
The output is then formed by aggregating the sampled values:
\begin{equation}
  \hat{F}_{c}(x,g)=\sum_{k=1}^{K} a_{c,k}(x,g)\,\tilde{F}_{c,k}(x,g).
  \label{eq:spec_dcn_agg}
\end{equation}
Eq.~\eqref{eq:spec_dcn_agg} highlights the main distinction from SpeIW.
The RE-DCN family remains a \emph{discrete deformable operator}: it starts from discrete channel anchors $c+r_k$ and uses interpolation only as a local off-grid evaluation rule.
By contrast, SpeIW parameterizes a continuous spectral warp and enforces a globally ordered mapping over the full spectral dimension.
This difference is illustrated in Fig.~\ref{fig:dcn_vs_speiw}.

\begin{figure}[t]
\centering
\includegraphics[width=\linewidth]{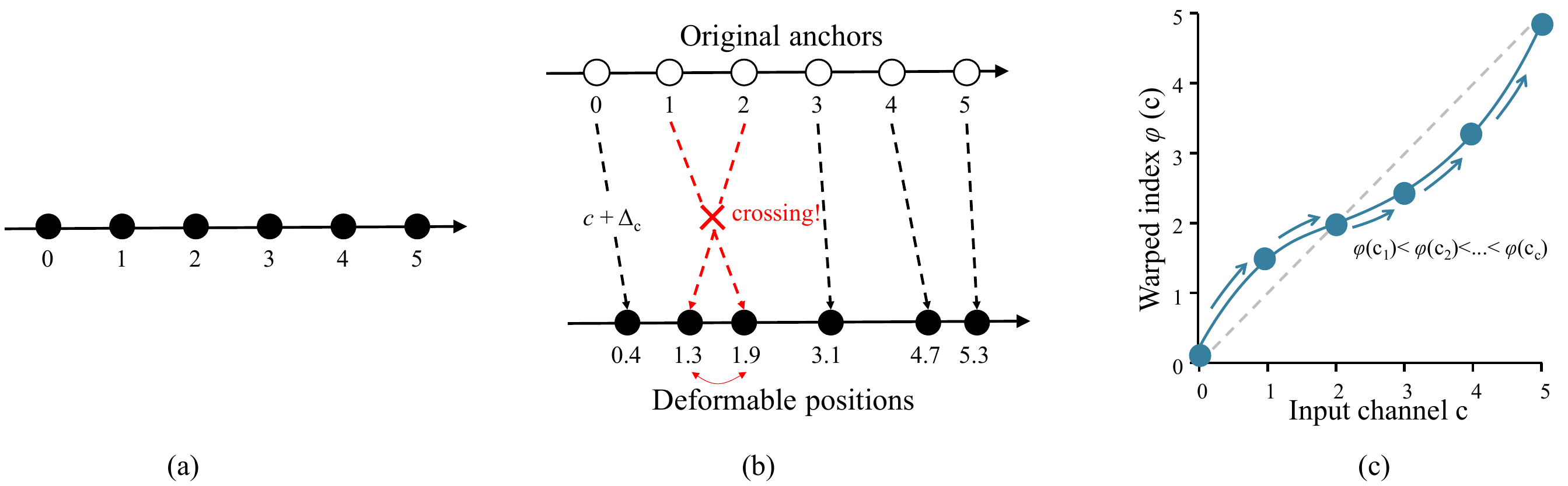}
\caption{Comparison of spectral resampling strategies. (a)~Regular spectral grid with uniform channel spacing. (b)~RE-DCN applies discrete offsets to each channel anchor; independent offsets can violate spectral ordering (red cross). (c)~SpeIW learns a continuous monotone warp that preserves spectral order by construction.}
\label{fig:dcn_vs_speiw}
\end{figure}

\subsection{Four RE-DCN Variants}
\label{sec:dcn_variants}

All four variants preserve the same backbone topology and differ only in how the aggregation weights $a_{c,k}(x,g)$ are parameterized.

\nbf{RE-DCN}
The first variant follows the original DCN formulation~\cite{dai2017deformable}.
Only the offsets are learned, and the aggregation weights reduce to the fixed spectral kernel weights:
\begin{equation}
  a_{c,k}(x,g)=w_k.
  \label{eq:re_dcn}
\end{equation}
Thus, RE-DCN performs offset-based discrete spectral sampling with shared kernel coefficients.

\nbf{RE-DCNv2}
Following DCNv2~\cite{Zhu_2019_CVPR}, we also predict a modulation scalar for each sampling point:
\begin{equation}
  a_{c,k}(x,g)=w_k\cdot m_{c,k}(x,g),
  \qquad
  m_{c,k}(x,g)=\sigma\!\bigl(\Gamma_{m,k}(F)(x,g)\bigr),
  \label{eq:re_dcnv2}
\end{equation}
where $\sigma$ is the sigmoid function.
Compared with RE-DCN, RE-DCNv2 allows the network to adjust the contribution of each deformed sampling location in an input-adaptive manner.

\nbf{RE-DCNv3}
Following the DCNv3 design used in InternImage~\cite{wang2023internimage}, we replace the modulated coefficients with input-adaptive normalized aggregation weights:
\begin{equation}
  a_{c,k}(x,g)
  =
  \frac{\exp\!\bigl(s_{c,k}(x,g)\bigr)}
       {\sum_{k'=1}^{K}\exp\!\bigl(s_{c,k'}(x,g)\bigr)},
  \label{eq:re_dcnv3}
\end{equation}
where $s_{c,k}(x,g)$ is predicted from the input feature.
This variant can be viewed as content-adaptive spectral aggregation with softmax-normalized importance over the deformed sampling points.

\nbf{RE-DCNv4}
Following DCNv4~\cite{xiong2024efficient}, we remove the softmax normalization and directly predict unconstrained aggregation weights:
\begin{equation}
  a_{c,k}(x,g)=\Gamma_{a,k}(F)(x,g).
  \label{eq:re_dcnv4}
\end{equation}
Compared with RE-DCNv3, this variant increases the flexibility of the dynamic aggregation by removing competition among sampling points.

Although these four operators differ in their weighting rules, they all share the same common structure:
predict spectral offsets from rotation-equivariant features, sample from off-grid spectral locations by interpolation, and aggregate the sampled responses.
They therefore provide a controlled comparison against SpeIW under the same rotation-equivariant backbone.

\nbf{Implementation details}
For all RE-DCN variants, we use $K\!=\!3$ sampling points with a symmetric stencil $\mathcal{R}=\{-1,0,+1\}$ centered at each spectral channel.
The offset predictor $\Gamma_{\Delta}$ is implemented as a single depthwise-separable equivariant convolution layer, matching the computational footprint of SpeIW's latent code extractor.
For RE-DCNv2/v3/v4, the additional weight predictors ($\Gamma_m$, $\Gamma_s$, $\Gamma_a$) share the same architecture.
All variants are integrated into the same SpeIWNet backbone by replacing only the spectral resampling module, and trained with identical hyper-parameters (Sec.~\ref{sec:impl_hyper}).
This ensures that performance differences in Tab.~6 of the main paper reflect the inductive bias of each spectral operator rather than architectural or optimization discrepancies.

\subsection{Why the RE-DCN Variants Remain Equivariant}
\label{sec:dcn_equiv}

The deformation in Eqs.~\eqref{eq:spec_dcn_offset}--\eqref{eq:spec_dcn_agg} acts only on the spectral dimension and is applied independently at each spatial location and orientation slot.
This is the key reason the operator remains rotation-equivariant.

\begin{proposition}[Spectral RE-DCN preserves strict $G$-equivariance]
\label{prop:redcn}
\mbox{}\par\noindent
Let $\hat{F}=\operatorname{RE\mbox{-}DCN}(F)$ denote any of the four variants above.
If the input group feature satisfies the regular transformation law
\[
  (\rho_h F)(x,g)=F(h^{-1}x,h^{-1}g), \qquad \forall h\in G,
\]
and the offset/weight predictors are built from equivariant layers with parameters shared across the group dimension, then
\begin{equation}
  \operatorname{RE\mbox{-}DCN}(\rho_h F)=\rho_h\bigl(\operatorname{RE\mbox{-}DCN}(F)\bigr),
  \qquad \forall h\in G.
  \label{eq:redcn_equiv}
\end{equation}
\end{proposition}

\textit{Proof.}
The predictors $\Gamma_{\Delta}$, $\Gamma_{m}$, and $\Gamma_{a}$ are built from equivariant layers with shared parameters across orientations, so their outputs inherit the regular representation (Sec.~\ref{sec:prelim}).
The spectral interpolation and aggregation in Eqs.~\eqref{eq:spec_dcn_interp}--\eqref{eq:spec_dcn_agg} act only along the spectral dimension at each fixed $(x,g)$; since these operations do not mix information across spatial or group dimensions, they commute with the group action.
Composing these operations yields Eq.~\eqref{eq:redcn_equiv}.

\nbf{Variant-specific remarks}
The proof applies to all four variants because their difference lies only in the aggregation weights $a_{c,k}$.
For RE-DCN, $a_{c,k}=w_k$ are fixed coefficients.
For RE-DCNv2, the modulation scalars are predicted from equivariant features and applied pointwise.
For RE-DCNv3/v4, the softmax (or its removal) operates only over the index $k$ at each $(x,g)$, commuting with the group action.

\nbf{Discussion}
The above proof also clarifies the difference between \emph{rotation equivariance} and \emph{spectral continuity}.
RE-DCN variants preserve rotation equivariance because the deformation is conditioned by equivariant features and the resampling is confined to the spectral dimension.
However, unlike SpeIW, they do \emph{not} impose a globally ordered or continuous warp across spectral channels.
Their inductive bias remains local and discrete: offsets are predicted around discrete channel anchors, and interpolation is used only to evaluate off-grid sample locations.
This difference in inductive bias explains the performance gap observed in Tab.~6 of the main paper: SpeIW retains equivariance while additionally modeling the spectral dimension as a continuous ordered field.

\begin{figure}[!t]
\centering
\includegraphics[width=\linewidth]{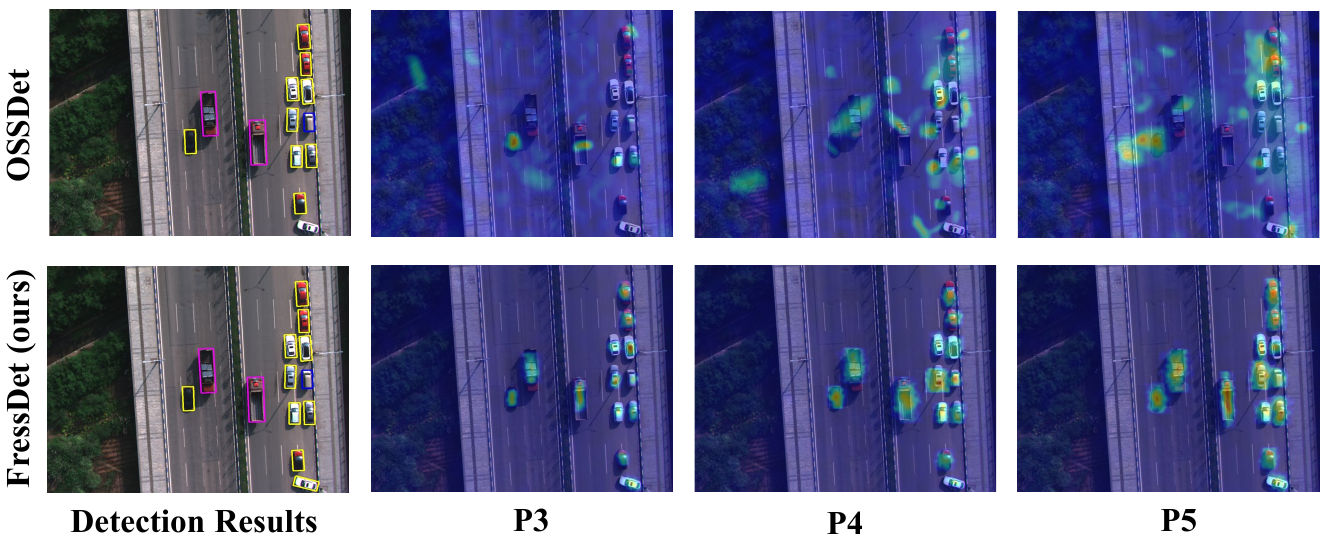}
\caption{Comparison of multiscale feature maps.}
\label{fig:multiscale}
\end{figure}

\begin{figure}[t]
\centering
\includegraphics[width=\linewidth]{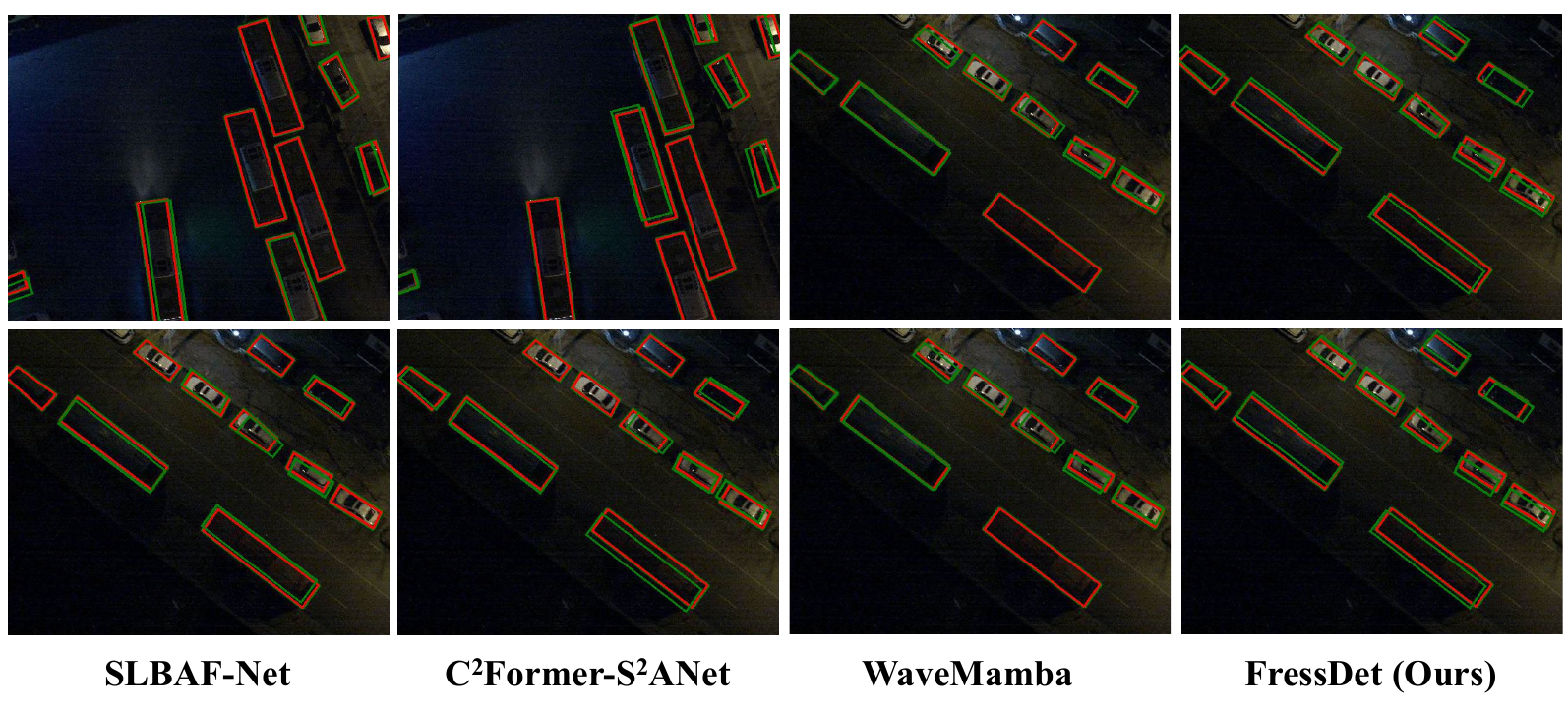}
\caption{Detection results on DroneVehicle. Red boxes indicate the ground truth, while green boxes denote the predictions.}
\label{fig:vis_dronevehicle}
\end{figure}

\section{Additional Visualizations}
\label{sec:vis}

\subsection{Multiscale Feature Visualization}
\label{sec:vis_multiscale}

We further visualize the multiscale feature maps of FressDet and competing methods, as shown in Fig.~\ref{fig:multiscale}.
FressDet propagates rotation-equivariant representations through the SpeIWNet backbone and refines them via ReCoW's adaptive spectral-spatial fusion, thereby features across scales (P3/P4/P5) exhibit consistent focus on target regions.
In contrast, baseline methods show scattered or inconsistent activations, particularly for small and densely packed objects.

\subsection{Qualitative Results on DroneVehicle}
\label{sec:vis_dronevehicle}

Fig.~\ref{fig:vis_dronevehicle} shows qualitative detection results on DroneVehicle.
FressDet accurately localizes vehicles with various orientations and scales, demonstrating robust performance in cluttered backgrounds where competing methods often produce false positives or miss small targets.

\subsection{Rotation Equivariance Visualization}
\label{sec:vis_equiv}

Fig.~\ref{fig:equiv} demonstrates the rotation equivariance property of FressDet by visualizing feature maps under input rotations of $0^\circ$, $90^\circ$, $180^\circ$, and $270^\circ$.
As the input image rotates, the corresponding feature maps rotate accordingly, confirming that FressDet maintains strict rotation equivariance throughout the network.

\noindent\begin{minipage}{\linewidth}
\centering
\includegraphics[width=0.94\linewidth]{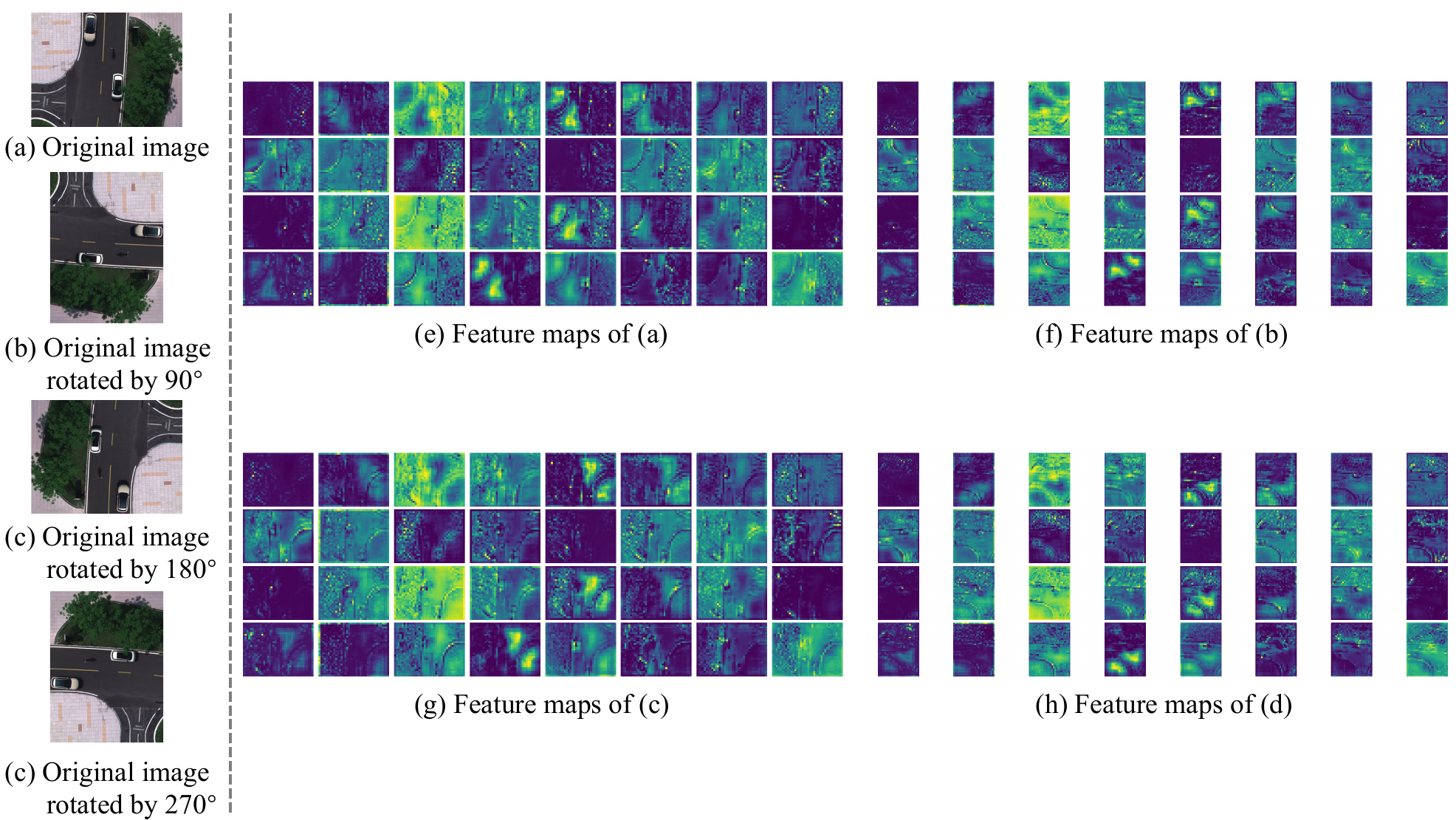}
\figcaption{Rotation-equivariance visualization of FressDet. The original image and its rotated versions at (a) $0^\circ$, (b) $90^\circ$, (c) $180^\circ$, and (d) $270^\circ$ produce correspondingly rotated feature maps in (e), (f), (g), and (h), showing the 32 output channels.}
\label{fig:equiv}
\end{minipage}

\clearpage

\putbib[main]
\end{bibunit}
\end{NoHyper}

\end{document}